\documentclass[11pt]{article}

\usepackage[margin=1in]{geometry}
\usepackage{setspace}
\usepackage{tikz}
\usepackage{hyperref}

\usepackage{natbib}
\usepackage{url} 
\usepackage{amsmath} 
\usepackage{amsfonts}
\usepackage{graphicx}
\usepackage{pifont}
\usepackage{algorithm}
\usepackage{algpseudocode}
\usepackage{tikz}
\usepackage[skip=5pt]{caption}

\newcommand{\NN}{N}
\newcommand{\x}{\mathbf{x}}
\newcommand{\y}{\mathbf{y}}

\newcommand{\A}{\mathcal{A}}

\newcommand{\K}{\mathcal{K}}
\newcommand{\s}{\mathbf{s}}
\usepackage{xcolor}
\usepackage{ulem}
\usepackage{colortbl}
\renewcommand{\emph}[1]{\textit{#1}}

\newcommand{\PrClass}{\mathcal{D}}
\newcommand{\Pars}{\mathbf{p}}
\newcommand{\parSpace}{\mathcal{P}}

\usepackage{amsthm}
\usepackage{bigstrut}
\usepackage{soul}
\usepackage{booktabs}
\usepackage{tabularx}
\usepackage{multicol}
\usepackage{multirow}
\usepackage{adjustbox}

\newtheorem{definition}{Definition}
\newtheorem{assumption}{Assumption}
\newtheorem{corollary}{Corollary}
\newtheorem{proposition}{Proposition}
\newtheorem{thm}{Theorem}

\newtheorem{remark}{Remark}

\usepackage{natbib}
\algtext*{EndIf}
\algtext*{EndFor}

\title{\vspace{-2.5cm}Zero-shot Generalization in Inventory Management:\\ Train, then Estimate and Decide}

\date{} 

\usepackage{authblk}
\usepackage{changepage}

\author[1]{Tarkan Temizöz\thanks{Corresponding author: \texttt{t.temizoz@tue.nl}}}
\author[1]{Christina Imdahl}
\author[1]{Remco Dijkman}
\author[1,2]{Douniel Lamghari-Idrissi}
\author[1]{Willem van Jaarsveld}

\affil[1]{Department of Industrial Engineering and Innovation Sciences, Eindhoven University of Technology, PO Box 513, Eindhoven 5600 MB, Netherlands
}
\affil[2]{ASML, 5504 DR, Veldhoven, The Netherlands}
\begin{document}

\maketitle
\vspace{-1.0cm}

\begin{adjustwidth}{-0.8cm}{-0.8cm}
\begin{abstract}
Deploying deep reinforcement learning (DRL) in real-world inventory management presents challenges, including dynamic environments and uncertain problem parameters, e.g. demand and lead time distributions. These challenges highlight a research gap, suggesting a need for a unifying framework to model and solve sequential decision-making under parameter uncertainty. We address this by exploring an underexplored area of DRL for inventory management: training generally capable agents (GCAs) under zero-shot generalization (ZSG). Here, GCAs are advanced DRL policies designed to handle a broad range of sampled problem instances with diverse inventory challenges. ZSG refers to the ability to successfully apply learned policies to unseen instances with unknown parameters without retraining.

We propose a unifying Super-Markov Decision Process formulation and the Train, then Estimate and Decide (TED) framework to train and deploy a GCA tailored to inventory management applications. The TED framework consists of three phases: training a GCA on varied problem instances, continuously estimating problem parameters during deployment, and making decisions based on these estimates. Applied to periodic review inventory problems with lost sales, cyclic demand patterns, and stochastic lead times, our trained agent, Generally Capable Lost Sales Network (GC-LSN) consistently outperforms well-known traditional policies when problem parameters are known. Moreover, under conditions where demand and/or lead time distributions are initially unknown and must be estimated, we benchmark against online learning methods that provide worst-case performance guarantees. Our GC-LSN policy, paired with the Kaplan-Meier estimator, is demonstrated to complement these methods by providing superior empirical performance.  
\end{abstract}
\end{adjustwidth}

\vspace{-0.9cm}
\section{Introduction}\label{sec:intro}
\vspace{-0.35cm}

The digital revolution, characterized by unprecedented availability of historical data and advances in computational capabilities, has led to a paradigm shift in operations management research toward data-driven methodologies \citep{misic2020}. For example, in inventory management, where simple and well-analyzed heuristics are widely adopted, there is a growing interest in developing and applying machine learning (ML) \citep[][]{qi2023deep} and deep reinforcement learning (DRL) algorithms \citep[][]{gijsbrechts2022can}. This evolution promises more sophisticated solutions to complex inventory practices where heuristics are still the norm \citep{boute2021deep}. 

However, deploying data-driven ML applications in real-life operations involving sequential decision-making presents significant challenges. One major challenge is the unpredictable, dynamic and constantly changing nature of operational environments \citep{amiri2023adaptive, Cheung2023nonstationary}, requiring trained policies to be robust and adaptable to similar yet unseen conditions during deployment \citep[cf.][]{Kirk2023SurveyZSGDRL}. For instance, global industries frequently encounter shifts in demand patterns and supply uncertainties, highlighting the limitations of relying on swathes of historical data and algorithms trained on them \citep{Gong2023}. 

A related challenge involves the decision-makers' (DMs) often incomplete understanding of the problem environment, henceforth referred to as the problem parameters or \emph{parameterization}, that influence state transitions. Here, the system's \emph{state} \cite[cf.][]{puterman2014markov} refers to the current status of all relevant variables at a given point in time, and \emph{state transitions} are the result of actions taken, leading to a new state. (For example, in the lost sales inventory system \cite[cf.][]{zipkin2008old}, the \emph{state transitions} correspond to the changes to on-hand and pipeline inventory and costs incurred in each period when a certain demand occurs and a certain order is placed. The problem \emph{parameters} influence state transition probabilities and costs, and include the applicable demand and lead time distributions, the holding cost and penalty costs.) 

Clearly, the optimal actions in a problem depend on the problem parameters. Thus, for effective deployment of ML/DRL algorithms in real-world settings, it is crucial to estimate and adapt to the true problem parameters. Unpredictable changes, such as shifts in demand distributions and the introduction of new products, can limit the applicability of historical data. Additionally, DMs may be constrained in accurately inferring the true parameters when demand is censored by stockouts and only sales data are observed \citep{huh2011adaptive}. Consequently, DMs are often compelled to re-estimate problem parameters based on recent observations, and such estimates must be frequently updated to avoid decision making based on outdated information.

In most applications of ML/DRL to inventory management, policies are trained for each specific set of parameters (or \emph{instance}) of the decision problem \citep[see e.g.][]{oroojlooyjadid2021deep,gijsbrechts2022can,maichle2024general,wang2024}. However, when deploying such models in a practical setting, the need to remodel and retrain a policy whenever the estimated problem parameters are updated represents a significant computational challenge. This challenge becomes even more pressing when considering that DMs typically need to make decisions across many tasks in each decision period, such as managing replenishments for thousands of products, and that training a single model can take hours \cite[see][]{gijsbrechts2022can}. 

These challenges reveal a significant research gap, indicating the need for a unifying framework to model and solve the sequential decision-making problems when problem parameters are uncertain and frequently updated. In this paper, we address this challenge by exploring an underexplored area of DRL for inventory management problems: training generally capable agents (GCAs) under zero-shot generalization (ZSG). In our context, \emph{GCAs} refer to advanced DRL policies that perform well across a broad range of sampled problem instances with diverse inventory challenges while \emph{ZSG} means that a GCA can be directly applied to new instances with unknown parameters that it was never explicitly trained on. Motivated by the desire to bring DRL applications closer to practical inventory applications, we propose a unifying MDP formulation and a solution framework to train and deploy a GCA tailored to inventory management problems. In particular, our contributions can be summarized as follows.

Firstly, to formalize the problem setting, we introduce a Super-Markov Decision Process (Super-MDP) representation for a class of inventory problems. Classical Markov Decision Processes (MDPs) \citep{puterman2014markov,Porteus2002} are traditionally used to model individual problem instances in inventory management. In contrast, a Super-MDP encompasses all MDP instances that can potentially arise from specific parameterizations of this class of problems, treating each instance as a particular member of a broader population of related MDPs. This population-based representation captures the DM’s sequential decision-making under uncertainty about which problem parameters apply.

Moreover, we propose the \emph{Train, then Estimate and Decide (TED)} framework to train GCAs based on the Super-MDP formulation and to enable their deployment in unknown environments. To \emph{train} a GCA, we observe that while problem parameters such as true demand and lead time distributions are unknown in typical inventory applications, DMs possess domain knowledge that typically allows them to determine a (possibly wide) \emph{range} in which the true problem parameters are expected to reside. We later call this range \emph{probable parameter space}. As a result, we can formulate a Super-MDP tailored to the resulting class of problem instances that the DM needs to address. To effectively \emph{train} a GCA for this Super-MDP, we sample problem instances from that class, and train a policy using DRL to select actions depending on the dynamic problem state as well as the problem parameters, i.e. we learn a \emph{single} policy that is applicable for all instances in the class. During the \emph{Estimate} phase, we continuously collect data while simultaneously estimating problem parameters, thus reducing the uncertainty caused by the policies' incomplete knowledge of the environment. For the \emph{Decide} phase, we note that a given set of problem parameter estimates yields a specific instance of an inventory problem, and we take decisions for that instance by deploying the trained GCA based on these parameter estimates without retraining (ZSG). 

The TED framework builds upon a rich tradition in \emph{classical} inventory management, where problem specific policies (e.g. base-stock policies, $(s,S)$ policies, capped base-stock policies etc.) are optimized based on \emph{estimated} problem parameters \citep[e.g. forecasted demand distribution, see][]{Silver2016,Axsater2015InventoryControl}, and the resulting decisions are deployed in practice. Training new DRL policies whenever parameter estimates are updated is prohibitively time-consuming especially in the case of many products, and the TED framework overcomes this limitation by training a \emph{single} policy (our GCA) for a wide range of parameters, and as such it is readily applicable without retraining after updating parameter estimates, enabling data-driven decision making for complex decision problems. 

To demonstrate the effectiveness of our TED framework, we apply it to a class of inventory problems involving periodic review, and further characterized by lost sales, and possibly by cyclic demand patterns and stochastic lead times. For this class of inventory problems, we train a GCA during the Train phase, named \emph{Generally Capable Lost Sales Network (GC-LSN)}. To validate it, GC-LSN is rigorously benchmarked versus the base-stock and capped base-stock policies on a wide range of problems assuming full availability of the parameters of each problem instance, both for GC-LSN and for the benchmarks. We find that GC-LSN consistently outperforms these well-performing benchmarks. Moreover, as part of the TED framework, we test GC-LSN under conditions where demand and lead time distributions are initially \emph{unknown} and must be estimated. For the Estimate phase, we adopt the non-parametric Kaplan-Meier estimator \citep{kaplan1958nonparametric} to estimate the demand distribution and construct a relative frequency distribution for the lead time distribution. Our experiments show that TED outperforms a range of benchmarks specifically designed for online learning \citep{lyu2024, Agrawal2022regret, zhang2020}. To our knowledge, TED stands uniquely as the only general-purpose algorithmic framework capable of addressing these diverse inventory challenges collectively, particularly when key information such as demand and lead time distributions are initially unknown and when data is limited and censored.

We make the implementation code of TED publicly available, including the trained model weights of our generally capable agent GC-LSN. This can be accessed through our GitHub repository\footnote{ \href{https://github.com/tarkantemizoz/DynaPlex}{https://github.com/tarkantemizoz/DynaPlex}}. Our objective in making the code and the trained agent available is twofold: academically, to support the principles of open science and enhance the reproducibility of reinforcement learning algorithms, as emphasized by \citet{Henderson_Islam_Bachman_Pineau_Precup_Meger_2018}; and managerially, to provide a state-of-the-art, ready-to-deploy lost sales inventory policy; practitioners facing similar inventory challenges may find this resource valuable.

This paper is organized as follows: \S\ref{sec:literature} reviews the literature, \S\ref{sec:prelim} provides preliminaries. We formulate the problem description in \S\ref{sec:prob} while \S\ref{sec:ted} introduces the TED framework. \S\ref{sec:inventorymodel} discusses the periodic review inventory control problem. Experimental setup and numerical results are presented in \S\ref{sec:numex}. \S\ref{sec:discussion} provides discussion and ablation experiments while \S\ref{sec:conclusion} concludes the paper.

\vspace{-0.35cm}
\section{Literature Review}\label{sec:literature}
\vspace{-0.35cm}

Our research intersects with three expanding bodies of literature. Initially, we explore the concept of training GCAs under \emph{ZSG}, the fundamental approach of this paper. Subsequently, we examine the prevalent ideas in \emph{online learning for inventory systems}, which represent the conventional methodology for addressing the leading application of this study: periodic review inventory control under conditions of lost sales with limited and censored data. Finally, we review \emph{ML for operations management}, specifically DRL, which serves as the main methodological tool within our TED framework for training GCAs.

\vspace{-0.35cm}
\subsection{Zero-shot generalization for reinforcement learning}
\vspace{-0.35cm}

ZSG refers to the ability of a trained model to perform well on tasks or parameter settings not encountered during training, without any task-specific fine-tuning. In supervised or unsupervised learning, this typically involves producing accurate predictions or representations for previously unseen classes or domains. In RL, however, ZSG entails training an agent on a set of source environments, each with its own dynamics and reward structure, and requiring it to execute effective policies in previously unseen target environments \emph{from the first decision step onward}, without further learning. This setting introduces challenges that do not arise in supervised or unsupervised learning: (\emph{i}) the sequential nature of decision-making, so early actions can create long-term consequences; and (\emph{ii}) key characteristics of the environment may be unknown or censored.

Following the taxonomy of \citet{Kirk2023SurveyZSGDRL}, we address the \emph{iid} ZSG setting, where training and testing environments are drawn from the same distribution, preparing agents for robust performance across diverse yet distributionally similar instances. Our work falls into the category where environment parameters are not directly observable and must be inferred from interactions, as in \citet{sodhani2022block}, but differs in both domain and structure: we focus on inventory management problems, where parameter estimation and mapping to problem instances are well-defined and relatively low-dimensional. This focus allows us to develop scalable, high-performing policies through our TED framework and Super-MDP formulation, enabling agents to act effectively across a wide range of parameterized problem instances without retraining.

Within the ZSG-RL literature, a common formalism is the \emph{Contextual Markov Decision Process} (CMDP) \citep{hallak2015contextual}, in which a context variable induces variation in the system’s dynamics and rewards. Our Super-MDP formulation can be viewed as a CMDP equipped with additional structure tailored to operations settings. In our case, the “context’’ is an interpretable parameter vector (e.g., demand, lead-time, and cost parameters), and each parameter vector defines a distinct MDP instance. The DM must manage a population of such instances, consisting of hundreds of products with heterogeneous characteristics. Importantly, we later endow the underlying parameter space with bounded ranges using domain knowledge and a task-specific metric, which together yield a “probable’’ parameter space and support a Lipschitz analysis of zero-shot generalization across instances. We refer to this population-based, metric CMDP view as a Super-MDP; it aligns naturally with inventory management and forms the basis for the theoretical and empirical developments in this paper.

Despite its importance, ZSG in inventory management remains relatively underexplored. Notable early work bridging this gap includes \citet{batsis2024}, \citet{akkerman2024solvingdualsourcingproblems}, \citet{vanderHaar2023}, and \citet{dehaybe2024deep} which suggest potential for developing GCAs in this area. A key distinction of our work from those applications in inventory management is the introduction of the unifying Super-MDP formulation of the problem instances, which enables modeling cases where true problem parameters are initially unknown. This uncertainty necessitates the Estimate phase in our TED framework, allowing for real-time parameter estimation and policy adaptation. In contrast, prior applications assume that all true parameters are known in advance.

\vspace{-0.35cm}
\subsection{Online learning for inventory systems} \label{sec:onlinelearninglit}
\vspace{-0.35cm}

In real-life scenarios, DMs frequently encounter unknown dynamics in demand and supply. As a result, a significant segment of literature has explored online learning strategies tailored to managing lost sales inventory control with non-zero lead times. This body of work has primarily focused on optimizing policy parameters (e.g. base-stock levels) in online learning, and has recently achieved breakthroughs in the form of worst-case regret bounds for parameter optimization within various heuristic policy classes, including results for base-stock policies \citep{Agrawal2022regret}, capped-base stock policies \citep{lyu2024}, and constant order policies \citep{chen2023supply}. Of course, adoption of heuristics limits decision makers to the performance of those heuristics for various cases. For instance, constant order policies perform poorly when the long-run average cost does not remain insensitive to lead times \citep{bai2023asymptotic}, while appropriate heuristics for the case of cyclic demand patterns are missing \citep{Gong2023}. 

Our study introduces a new approach to real-time decision making in which the policy is learned offline and applied without further learning during execution. Indeed, our GC-LSN is trained prior to being confronted with any specific instance: we adopt a unified Super-MDP formulation and the TED framework for training GCAs. Within TED, we propose to subsequently feed the trained policy network with a relatively straightforward parameter estimation procedure (i.e., based on the Kaplan-Meier estimator \citep{kaplan1958nonparametric}), and focus on empirical performance in numerical experiments. Our approach is demonstrated to complement the online learning literature with a method that performs very well empirically under a wide range of assumptions. The resulting general policy provides a unique tool for benchmarking a variety of online inventory control strategies for a broad range of problems, both in cases where there are benchmarks with worst-case performance guarantees available, and in cases where such benchmarks are not yet available. 

\vspace{-0.35cm}
\subsection{Machine learning in operations management}
\vspace{-0.35cm}

ML has extensive applications across various fields of operations management, including supply chain management, revenue management, and healthcare \citep{misic2020}. Typically, ML approaches adhere to the Predict-Then-Optimize (PTO) paradigm, where ML models forecast future scenarios, such as customer demands, which are then used to inform optimization decisions \citep{ElmachtoubGrigas2021}. However, forecasting becomes particularly challenging in environments where the future is highly stochastic and decisions have far-reaching consequences \citep{sinclair2023}.

An alternative to the PTO paradigm is DRL, which seeks to optimize decisions directly \citep{boute2021deep}. DRL has been effectively applied in various domains, including multi-echelon supply chains \citep{harsha2021math, oroojlooyjadid2021deep}, lost sales inventory control \citep{gijsbrechts2022can, temizoz2025deep}, and large-scale online retail operations \citep{madeka2022deep, andaz2024learning, Liu2023ai}. Nevertheless, the effectiveness of DRL often hinges on the availability of accurate uncertainty distributions \citep[see, e.g.,][]{temizoz2025deep} or a substantial dataset of historical transactions that are applicable for the task to be learned \citep[see, e.g.,][]{madeka2022deep}. In scenarios where data is limited and/or censored, and in cases where problem parameters are not fully observable, the utility of the previous DRL applications may diminish. 

The TED framework aims to fill this gap by expanding the application of DRL to include settings with scarce and censored data, thereby broadening the utility of these algorithms in practical operations management.

\vspace{-0.35cm}
\section{Preliminaries}\label{sec:prelim}
\vspace{-0.35cm}

This section introduces the fundamental concepts and notation necessary to define the problem setting and the TED framework presented in the subsequent sections. We begin with an introduction to MDPs, a common modeling framework for sequential decision-making \citep{puterman2014markov}.
\begin{definition}[Markov Decision Process]\label{def:mdp}
A Markov Decision Process is represented by the tuple $\mathcal{M} = \langle \mathcal{S}, \mathcal{A}, f, C, \s_{0} \rangle$. Here, $\mathcal{S}$ and $\mathcal{A}$ are finite sets of states and actions, respectively. The transition function $f: \mathcal{S} \times \mathcal{A} \to \mathcal{S}$ maps each state-action pair $(\s, a)$ to a subsequent state $\s'$. The cost function $C: \mathcal{S} \times \mathcal{A} \to \mathbb{R}$ assigns a non-negative and bounded cost $c = C(\s, a)$ to each state-action combination, while the initial state $\s_{0}$ is either deterministically or randomly set.
\end{definition}

Many common operational problems can be modeled as MDPs. For instance, consider a single-item periodic-review lost sales inventory control problem with zero lead time. In this system, the current inventory position represents the state, the order decision constitutes the action, and state transitions are determined by external demand realizations and the chosen actions. However, the dynamics and outcomes of such systems are governed by external parameters, which are typically assumed to be static and fixed within the MDP framework. In this context, the dynamics and outcomes are influenced by the demand distribution, the lead time distribution, the holding cost, and the penalty cost for unmet demand. We refer to all these factors as the parameterization of the MDP, or more generally, as the parameterization of the decision problem:
\begin{definition}[Parameterization and Parameter Space]
The \textbf{parameterization} of a problem consists of any structured exogenous parameters that completely determine the system's dynamics and outcomes. The parameter space is the set of all possible parameterizations of a decision problem, denoted as $\parSpace$. It comprises each combination of parameters $\Pars$, with $\Pars \in \parSpace$. 
\end{definition}

The parameters may take many forms, including fixed numerical constants (penalty and holding costs) and the probability distributions of exogenous stochastic processes (demand and lead time distributions). They can either be directly observed or estimated based on data. For clarity, we denote an MDP associated with a specific parameterization as $\mathcal{M}^{\Pars} = \langle \mathcal{S}^{\Pars}, \mathcal{A}^{\Pars}, f^{\Pars}, C^{\Pars}, \s_{0}^{\Pars} \rangle$ and refer to such an MDP as a \emph{problem instance}. To simplify notation, we often omit explicit references to this parameter dependence, even though every MDP is inherently tied to a specific parameterization.

A policy $\pi^\Pars$ for an MDP $\mathcal{M}^\Pars$ with parameterization $\Pars$ is defined as a function that maps states to actions, i.e., $\pi^\Pars: \mathcal{S}^\Pars \to \mathcal{A}^\Pars$. The optimal policy $\pi^{\Pars*}$ for $\mathcal{M}^\Pars$ minimizes the undiscounted average cost over an infinite horizon, which aligns with steady-state evaluation in classical inventory management \citep[see][]{puterman2014markov}:
\[
\pi^{\Pars*} = \arg\min_{\pi^\Pars} \limsup_{T \to \infty} \frac{1}{T} \mathbb{E} \left[ \sum_{t=0}^{T-1} C^\Pars(\s_t, \pi^\Pars(\s_t)) \right],
\]
where $\s_t$ is the state at time $t$ and $\pi^\Pars(\s_t)$ is the action taken in state $\s_t$ according to policy $\pi^\Pars$. The expectation $\mathbb{E}$ accounts for the stochastic processes governing the state transitions. Throughout this paper, finding policy $\pi^{\Pars*}$ for $\mathcal{M}^\Pars$ is referred to as solving a task with a parameterization $\Pars$. While the solution methodologies for finding $\pi^{\Pars*}$ are vast \citep{puterman2014markov}, we are interested in finding a single policy (a GCA $\pi_S$, such that $\pi_S: \mathcal{S} \times \parSpace \to \mathcal{A}$) that is robust to changes in $\Pars$, and especially one that can perform well when $\Pars$ is unknown. 

\vspace{-0.35cm}
\section{Problem Description}\label{sec:prob}
\vspace{-0.35cm}

In this section, we formally define the problem setting of this paper. Let $\PrClass$ represent a class of sequential decision-making problems, influenced by a parameter space $\parSpace$. Within $\PrClass$, $n$ independent and distinct tasks must be managed, each requiring the solution of an MDP ($\mathcal{M}^\Pars$) tailored to its respective parameterization $\Pars \in \parSpace$. These parameters may be subject to external factors and can change over time. For instance, replenishment decisions may need to be made for multiple products over a designated time period, where the demand and/or lead time distributions can be censored and non-stationary. Consequently, three common challenges (decision contexts) may arise when managing the sequential decision-making problems within $\PrClass$:

\begin{itemize}
\item \textit{(Scalability-1)} Decisions must be made for $n$ independent tasks, each characterized by different parameterizations. This necessitates solving at least $n$ distinct MDPs, each tailored to its specific parameterization, and employing $n$ different decision-making policies. As $n$ increases, this approach becomes time-consuming and impractical.

\item \textit{(Non-stationarity-2)} The parameters within a specific task may evolve over time, requiring the solution of new MDPs for updated parameter combinations. For example, an increase in product demand could render the existing replenishment strategy ineffective, or new tasks with unknown parameterizations may emerge through the introduction of new products. Here, a parameterization may itself encode nonstationary dynamics (e.g., cyclic demand), and its numerical values may subsequently evolve to a different element of $\parSpace$. Thus, nonstationarity as a decision context refers to the possibility that the true process, and therefore the effective MDP, changes over time.

\item \textit{(Obscurity-3)} The true parameters may not be directly observable and must instead be inferred from potentially limited and censored real-time data. For instance, while the actual demand distribution for a product remains unknown, an empirical distribution constructed from observed sales data might be used to inform replenishment decisions. As more sales data becomes available, estimates of the distribution must be updated, requiring the MDPs to be solved with respect to estimated parameters.

\end{itemize}

Across these three decision contexts, managing many tasks, adapting to evolving parameters, and operating under incomplete information, the underlying challenge is that new or updated parameterizations are repeatedly encountered. Addressing each case by retraining a separate policy for every resulting MDP $\mathcal{M}^\Pars$ is computationally expensive and often impractical. This underscores the importance of developing a single GCA with ZSG capability, enabling effective decision-making without retraining whenever parameter estimates change. We therefore begin by specifying the rationale that governs how decisions are made at each decision epoch.

\begin{assumption}[Optimization Rationale]\label{as:1} 
We do not know when or how the parameterization will change, and cannot utilize probabilistic information to anticipate and optimize for potential parameter changes accordingly. As a result, we optimize the policy assuming the stationarity of the problem parameters, pursuing optimal decision-making as long as the problem's parameters remain unchanged.
\end{assumption}

Assumption \ref{as:1} enables the decomposition of the problem into a collection of independent and stationary MDPs, each characterized by a distinct parameterization.\footnote{Even if a parameterization induces nonstationary dynamics, it can still define a stationary (time-homogeneous) MDP when the nonstationarity is structured and incorporated into the state space \citep[Sec.~2.3]{kallenberg2011markov}.} This approach is common in inventory management, where optimization strategies are typically based on current estimates of problem parameters \citep{Axsater2015InventoryControl}. In fact, it is akin to a rolling horizon strategy: the parameterization is assumed to be stationary for the time being, while at each time step, changes in parameters are monitored, and values are potentially re-estimated if they are unknown.

The ability to decompose the problem $\PrClass$ into independent and stationary MDPs motivates our definition of the Super-MDP. The Super-MDP can be understood as a population of all MDPs related to our decision problem $\mathcal{D}$ and thus formally defines the problem $\PrClass$ as follows:

\begin{definition}[Super-Markov Decision Process] A Super-Markov Decision Process is defined by the tuple $\mathcal{M}_S = (\parSpace, \mathcal{S}, \mathcal{A}, \mathcal{H}, \mathcal{F})$, where: \begin{itemize} \item $\parSpace$ represents the parameter space, containing the true problem parameterizations $\Pars \in \parSpace$. \item $\mathcal{S}$ denotes the finite state space, encompassing all possible states $\s \in \mathcal{S}$. \item $\mathcal{A}$ indicates the finite action space, where each action $a \in \mathcal{A}$, and $\mathcal{A} = \{0, 1, \ldots, m\}$. \item $\mathcal{H}$ is distribution over the parameter space $\parSpace$, generating problem parameterizations $\Pars \sim \mathcal{H}$. \item $\mathcal{F}$ is the mapping function that relates each parameterization $\Pars$ to the corresponding elements of $\mathcal{M}^{\Pars}$ (i.e., $f^{\Pars}$, $C^{\Pars}$, $\s_{0}^{\Pars}$, Definition \ref{def:mdp}). Both the state space and action space are defined universally across parameterizations, such that $\mathcal{S}^\Pars \subseteq \mathcal{S}$ and $\mathcal{A}^\Pars \subseteq \mathcal{A}$. \end{itemize} \end{definition}

To operationalize the Super–MDP, we next clarify which aspects of the parameterization are known and which remain unknown. We formalize this distinction with the following assumption.
\begin{assumption}[Structural knowledge of the parameterization and compact parameter space]\label{as:paramspace}
We know the identity and structural role of all components that constitute the parameterization $\Pars \in \parSpace$ of the problem class (i.e., there are no hidden parameters), while the true numerical values of the parameters may be unknown and time-varying. Moreover,  sufficient domain knowledge exists to specify a priori bounds for each component.
\end{assumption}
This assumption is reasonable because the operational environment is typically well understood at a higher level, parameters are often naturally bounded (e.g., non-negativity of demand), and such assumptions are commonly employed in the online learning literature \citep[see][and the references therein]{lyu2024}. We later define and utilize a probable parameterization space and establish a parameter sampling function to conceptualize the Super-MDP.

The definition of the Super-MDP aligns with the objective of training a GCA. If a policy is a GCA, it should generate effective actions for any problem instance generated by $\mathcal{H}$ from the parameter space. A policy for a Super-MDP, denoted as $\pi_S$, can be defined as a function from state-parameterization pairs to actions, $\pi_S: \mathcal{S} \times \parSpace \to \mathcal{A}$. Our objective is to identify a jointly optimal policy $\pi^*_S$, defined as:
\[
\pi_S^* = \arg\min_{\pi} \mathbb{E}_{\Pars \sim {\mathcal{H}}}\left[ \limsup_{T \to \infty} \frac{1}{T} \mathbb{E} \left[ \sum_{t=0}^{T-1} C^{\Pars}(\s_t, \pi(\s_t,\Pars)) \right]\right].
\]

Since finding the optimal policy for even moderately sized MDPs is intractable, we must rely on approximation methods. Let $\bar{C}_\pi$ denote the expected per-period costs of policy $\pi_S$ for a Super-MDP over parameterizations $\Pars \in \parSpace$ with $\Pars \sim \mathcal{H}$. With a GCA, we aim to achieve an approximately optimal policy $\hat{\pi}_S$ such that $\bar{C}{\pi^*_S} + \epsilon \geq \bar{C}{\hat{\pi}_S}$, where $\epsilon$ is the approximation error. To achieve such a policy, we develop the TED approach, which we detail next.

\vspace{-0.35cm}
\section{Train, then Estimate and Decide}\label{sec:ted}
\vspace{-0.35cm}

Achieving an approximately optimal policy $\hat{\pi}_S$ for a Super-MDP requires a solution approach capable of simultaneously addressing MDPs with varying parameterizations, thereby obtaining a GCA. Our primary strategy involves separating the training and deployment phases of this policy, ensuring that it does not require retraining when the problem's parameterization changes, thus achieving ZSG. We propose such a solution within our \emph{Train, then Estimate and Decide-TED} framework.

In the \textit{Train} phase, we first construct the Super-MDP for the problem class $\PrClass$. Since the true parameter space $\parSpace$ and the distribution over the parameter space $\mathcal{H}$ are unknown, we define a probable parameter space $\hat{\parSpace}$ and an associated distribution $\hat{\mathcal{H}}$. The policy (GCA) is parameterized as a neural network, and we sample from the probable parameter space for training (see Figure \ref{fig:TED}, left part of the Train box) using $\hat{\mathcal{H}}$. The policy is trained for each sampled parameterization under the assumption that the parameterization remains fixed during training (see Figure \ref{fig:TED}, right part of the Train box). The parameterization of the problem is incorporated into the policy (the neural network) as input features, enabling the neural network to generalize to unseen parameterizations during the deployment phase, which comprises the \textit{Estimate} and \textit{Decide} steps.

\begin{figure}[ht]
    \centering
    \resizebox{\textwidth}{!}{
    \begin{tikzpicture}

\draw[thick] (-1.4,0) rectangle (8,5); 
\draw[thick] (13,3.5) rectangle (16,5) node[pos=.5] {$\hat{\Pars}_t = \mathcal{Y}(\mathcal{O}_t)$  }; 
\draw[thick, white] (9.1,0) rectangle (12,1.5) node[pos=.5, black] {\textbf{\shortstack{\footnotesize Generally capable\\ \footnotesize agent $\hat{\pi}_S$ }}}; 
\draw[thick] (13,0) rectangle (16,1.5) node[pos=.5] {$a_t = \hat{\pi}_S(\s_t, \hat{\Pars}_t)$ };

\node[align=center] at (3.3,5.5) {\textbf{T}rain};
\node[align=center] at (14.5,5.5) {\textbf{E}stimate};
\node[align=center] at (14.5,2) {\textbf{D}ecide};

\draw[->, thick] (8,0.75) -- (9,0.75);
\draw[->, thick] (12.1,0.75) -- (12.9,0.75);
\draw[->, thick] (15,3.75) to[out=10,in=-10] (15,1.25);
\draw[->, thick] (14,1.25) to[out=170,in=-90] (13.35,2.4);
\draw[->, thick] (13.35,2.6) to[out=90,in=-170] (14,3.75);
\node[rotate=0] at (13,2.2) {\shortstack{\scriptsize $t$}};
\node[rotate=0] at (12.8,2.7) {\shortstack{\scriptsize $t+1$}};

\draw[dotted, thick] (12.9,2.5) -- (13.75,2.5);

\draw[thick] (0,0.5) rectangle (1.5,3.5);
\node[rotate=90] at (-0.5,2) {\shortstack{\footnotesize Part of the\\ \footnotesize Super-MDP $\mathcal{M}^{\hat\parSpace}$}};
\node at (0.75,3) {\footnotesize $\mathcal{M}^{p_1}$};
\node at (0.75,2) {\footnotesize $\mathcal{M}^{p_2}$};
\node at (0.75,1) {\footnotesize $\mathcal{M}^{p_3}$};

\node[align=center] at (0.75,4.5) {\footnotesize Parameterizations for training,\\ \footnotesize sampled from $\hat{\parSpace} \supset P$};
\node[align=center] at (5.1,4.5) {\footnotesize Training the agent on\\ \footnotesize the parameterized MDPs};

\node at (3,3.0) {\footnotesize $s_t$};
\draw[->, thick] (3.2,3.0) -- (4,3.0) node[midway, above] {\footnotesize $f^{p_1}$};
\node at (4.4,3.0) {\footnotesize $s_{t+1}$};
\draw[->, thick] (4.8,3.0) -- (5.6,3.0) node[midway, above] {\footnotesize $f^{p_1}$};
\node at (6.0,3.0) {\footnotesize $s_{t+2}$};
\draw[->, thick] (6.4,3.0) -- (7.2,3.0) node[midway, above] {\footnotesize $f^{p_1}$};
\draw[dotted, thick] (7.4,3.0) -- (7.85,3.0);

\node at (3,2.0) {\footnotesize $s_t$};
\draw[->, thick] (3.2,2.0) -- (4,2.0) node[midway, above] {\footnotesize $f^{p_2}$};
\node at (4.4,2.0) {\footnotesize $s_{t+1}$};
\draw[->, thick] (4.8,2.0) -- (5.6,2.0) node[midway, above] {\footnotesize $f^{p_2}$};
\node at (6.0,2.0) {\footnotesize $s_{t+2}$};
\draw[->, thick] (6.4,2.0) -- (7.2,2.0) node[midway, above] {\footnotesize $f^{p_2}$};
\draw[dotted, thick] (7.4,2.0) -- (7.85,2.0);

\node at (3,1.0) {\footnotesize $s_t$};
\draw[->, thick] (3.2,1.0) -- (4,1.0) node[midway, above] {\footnotesize $f^{p_3}$};
\node at (4.4,1.0) {\footnotesize $s_{t+1}$};
\draw[->, thick] (4.8,1.0) -- (5.6,1.0) node[midway, above] {\footnotesize $f^{p_3}$};
\node at (6.0,1.0) {\footnotesize $s_{t+2}$};
\draw[->, thick] (6.4,1.0) -- (7.2,1.0) node[midway, above] {\footnotesize $f^{p_3}$};
\draw[dotted, thick] (7.4,1.0) -- (7.85,1.0);



\end{tikzpicture}
    }
    \caption{\textbf{Train}, then \textbf{Estimate} and \textbf{Decide} framework for solving sequential decision-making problems with dynamic parameter estimation. In each decision period, if parameters are unknown, the TED framework first estimates the parameterization (Estimate) from the observations and then determines the action (Decide) using the trained GCA. If parameters are already known, the GCA can be applied directly without the Estimate step.}
    \label{fig:TED}
\end{figure}

In the \textit{Estimate} phase, we estimate the parameterization $\hat{\Pars}_t$ of the actual system at time $t$ using a function $\mathcal{Y}$ based on the observations collected up to time $t$ ($\mathcal{O}_t$). In the \textit{Decide} phase, this parameterization estimate is fed into the pre-trained policy along with the current state. The policy then outputs the action to be taken at time $t$ (see the right part of Figure \ref{fig:TED}).

The main novelty of this approach is that the training step, during which the policy is optimized, precedes the estimation step. This contrasts with current prevalent approaches, which first estimate parameters and then optimize based on parameter estimates; see \citet{Agrawal2022regret, lyu2024, chen2023supply} for similar applications covering online learning in inventory management. We provide detailed descriptions of the different phases in the subsequent sections.

\vspace{-0.35cm}
\subsection{Train}\label{sec:train}
\vspace{-0.35cm}

The Train phase occurs before observing data and making decisions. This phase aims to develop a GCA that can make effective decisions for (unseen) problem instances in $\parSpace$ without further training (ZSG). To achieve this, the phase involves the formulation of a Super-MDP and the implementation of a learning algorithm to train the GCA for the formulated Super-MDP. We identify key challenges and conditions for both components to enhance the GCA's ability to generalize to unseen instances.

\vspace{-0.35cm}
\subsubsection*{Construction of the Super-MDP for training}
\vspace{-0.35cm}

We initiate this phase by first defining a Super-MDP for the designated problem class $\PrClass$. As outlined in the problem description, the true parameters may be obscure and unknown, making both the parameter space $\parSpace$ and the probability distribution $\mathcal{H}$ uncertain. According to Assumption~2, we know each element of $\Pars$ and can bound $\parSpace$ using domain knowledge. This allows us to construct a \emph{probable} Super-MDP whose parameter space $\hat{\parSpace}$ is a superset of the true space, i.e., $\parSpace \subseteq \hat{\parSpace}$.

\begin{remark} To train a GCA for the Super-MDP $\mathcal{M}_S = (\parSpace, \mathcal{S}, \mathcal{A}, \mathcal{H}, \mathcal{F})$, we employ a more general Super-MDP $\hat{\mathcal{M}}_S = (\hat{\parSpace}, \mathcal{S}, \mathcal{A}, \hat{\mathcal{H}}, \mathcal{F})$, where $\hat{\parSpace}$ is the probable parameter space that fully contains the true parameter space, i.e., $\parSpace \subseteq \hat{\parSpace}$, and $\hat{\mathcal{H}}$ is the probable parameter sampling distribution ensuring uniform coverage of $\hat{\parSpace}$. \end{remark}

While deriving an analogous object for $\mathcal H$ is harder, we explicitly define $\hat{\mathcal H}$ during training to sample  parameterizations and explain why broad, uniform coverage of $\hat{\parSpace}$ is desirable. Figure \ref{fig:parameterspace} illustrates an example of constructing $\hat{\parSpace}$ and $\hat{\mathcal{H}}$, where $\hat{\parSpace}$ is smooth and bounded, and $\hat{\mathcal{H}}$ samples parameterizations uniformly from $\hat{\parSpace}$. 

\begin{figure}[ht]
  \centering
  \resizebox{0.7\linewidth}{!}{\begin{tikzpicture}[scale = 0.9]

\fill[fill=gray!20, draw = black] (0, 0) rectangle (8, 8);

\fill[fill=blue!20, draw = black] (5.5, 2.9) circle (1.8);

\foreach \x/\y in {1/1, 1/6, 1.2/4, 2.5/4.5, 4.5/7, 6/7, 5.5/6.1,
2/2, 3/3, 4/1,  7.2/4.4, 2/7, 3/6.2, 4/5,  7/1, 6/1, 7/3,
4/2.5,  4.4/3.35, 5.3/1.75,  6.3/2.25,  5.8/3, 6.4/4, 4.9/4,} {
    \fill[gray, draw = black] (\x, \y) circle (0.2);
}

\foreach \x/\y in {4.5/3, 4.8/2.5, 6.1/1.45, 5.35/3.5, 5.7/4, 5.5/2.4, 6.9/3.7, 6.3/2.9} {
    \fill[blue, draw = black] (\x, \y) circle (0.2);
}

\node[anchor=west] at (9, 4) {\shortstack[c]{Probable parameter\\ space $\hat\parSpace$}};
\draw[->] (9,4) -- (7.5,4);
\node[anchor=west] at (9, 2.4) {\shortstack[c]{True parameter\\ space $\parSpace$}};
\draw[->] (9,2.4) -- (6.8,2.4);

\fill[gray, draw = black] (9, 7.5) circle (0.2);
\node[anchor=north west] at (9.5, 8) {\shortstack[l]{Parameterization generated from\\probable parameter generation function}};
\fill[blue , draw = black] (9, 6) circle (0.2);
\node[anchor=west] at (9.5, 6) {\shortstack[l]{Parameterization generated from\\ true parameter generation function}};

\end{tikzpicture}}
  \caption{Construction of $\hat{\parSpace}$ and $\hat{\mathcal H}$. As more parameterizations are sampled from $\hat{\mathcal H}$, the training set covers $\hat{\parSpace}$ more densely; since $\parSpace\subseteq\hat{\parSpace}$, this increases the likelihood that unseen test instances lie near trained ones, improving ZSG.}
  \label{fig:parameterspace}
\end{figure}
The construction of the probable parameter space $\hat{\parSpace}$ and the probability distribution $\hat{\mathcal{H}}$ defined over it may influence the generalization performance of the GCA. By defining $\hat{\parSpace}$ as a superset of the true parameter set $\parSpace$, reasonably similar parameterizations are guaranteed to be available for sampling during training. This allows the agent to encounter a broad spectrum of representative problem instances. The rationale becomes particularly evident under the Lipschitz conditions for Super-MDPs: when these are satisfied, small perturbations in parameterization induce proportionally small changes in transition dynamics and costs; thus, controlling the effect of parameter variation.

\begin{definition}[Lipschitz Super-Markov Decision Processes]\label{def:lip}
Let $\hat{\mathcal{M}}_S = (\hat{\parSpace}, \mathcal{S}, \mathcal{A}, \hat{\mathcal{H}}, \mathcal{F})$ be a given Super-Markov Decision Process, and $d(\cdot,\cdot)$ a generic metric on the parameter space. We call $\hat{\mathcal{M}_S}$ a \emph{Lipschitz Super-Markov Decision Process} with smoothness factors $L_f,L_r$ if for any two parameterizations $\Pars_i,\Pars_j \in \hat{\parSpace}$, 
\[
\forall (\s,a): \ \big|C^{\Pars_i}(\s,a)-C^{\Pars_j}(\s,a)\big| \le L_r\, d(\Pars_i,\Pars_j), \ \big\|f^{\Pars_i}(\cdot\mid \s,a)-f^{\Pars_j}(\cdot\mid \s,a)\big\|_1 \le L_f\, d(\Pars_i,\Pars_j),
\]
where $\|p-q\|_1 := \sum_{x\in\mathcal{S}} |p(x)-q(x)|$ is the $\ell_1$ distance between the two transition probability distributions over the next states for the same $(\s,a)$.
\end{definition}

Definition \ref{def:lip} can be interpreted as follows: when two parameterizations are close to each other (i.e., the distance $d(\Pars_i,\Pars_j)$ is small), the resulting differences in state transitions and costs under the same state–action pair will also be small. Larger parameter differences may lead to larger deviations, but these remain controlled by the Lipschitz constants $L_f$ and $L_r$ (cf.\ Definition~\ref{def:lip}).

The Lipschitz property provides actionable guarantees for the ZSG capabilities of a GCA. Specifically, the Lipschitz conditions impose a smoothness constraint on how transition probabilities and reward functions vary with the problem parameters. This smoothness supports predictable and continuous system behavior as parameters change. When these conditions hold, an agent that learns a decision-making policy from collected problem instances within a parameter space is better positioned to anticipate and perform well for closely related yet unseen instances (we provide a quantitative analysis in \S\ref{sec:est} when the true parameters are unknown and should be estimated). In this case, extrapolation of learned behaviors to new instances is facilitated by the Lipschitz guarantee that dynamics and rewards will not differ drastically, thereby reducing the need for additional training. Consequently, a policy effective in one instance is more likely to succeed in similar ones, strengthening the GCA’s performance in a ZSG setting \citep[see][for a related discussion]{Xu2012RobustnessAndGeneralization}. 

In many inventory settings, once the parameter space is bounded and costs are linear, and when key parameters are discrete or categorical (e.g., integer lead times encoded one–hot, crossover flags for order arrivals), the resulting family already satisfies the Lipschitz property: there exist finite $L_f$ and $L_r$ (cf.\ Definition~\ref{def:lip}). Non-Lipschitz behavior can typically arise from step or threshold costs (e.g., “incur $M$ if a condition holds”) that vary across instances. In such cases, Lipschitz property can be preserved by discretizing the threshold parameter; see Appendix~\ref{sec:keeplip} for a concrete example with a varying threshold.  

It is important to emphasize, however, that the Lipschitz property alone does not guarantee ZSG, and ZSG may still be observed empirically for non-Lipschitz Super-MDPs, but in such cases, no formal guarantees can be established, such as the one we derive in \S\ref{sec:est}. The realized test-time performance gap also depends on (i) the extent to which $\hat{\mathcal P}$ is adequately covered by training samples from $\hat{\mathcal H}$, (ii) the accuracy of parameter estimation at deployment, and (iii) the degree of optimization or approximation error in the learned policy. Poor coverage (e.g., a narrowly concentrated $\hat{\mathcal H}$), inaccurate parameter estimates, or insufficient training can all degrade performance even if the underlying Super-MDP is Lipschitz.

\vspace{-0.35cm}
\subsubsection*{The algorithm and training}
\vspace{-0.35cm}

We now explore algorithmic strategies for training GCAs under ZSG and utilize DRL as our primary methodological tool due to its proven success in handling large-scale sequential decision-making problems. According to \citet{Kirk2023SurveyZSGDRL}, supported by numerous references in their paper, methods addressing ZSG challenges in DRL can be categorized into three main areas: DRL-specific improvements, enhancing similarity between training and testing data, and addressing differences between training and testing environments. We begin by examining DRL-specific strategies for algorithm adoption.

The DRL domain has advanced significantly since the seminal introduction of the Deep Q-Network (DQN) algorithm by \citet{mnih2013playing}. However, designing algorithms with ZSG capability remains challenging, as highlighted by \citet{Kirk2023SurveyZSGDRL}. They identify several promising directions to enhance ZSG, including: (i) iterative policy distillation into reinitialized networks to mitigate non-stationarity (i.e., the shifting data distribution caused by the evolving policy during training), (ii) architectural separation of policy and value networks for more targeted optimization, and (iii) model-based approaches to reduce sample complexity and improve generalization.

Given these considerations, the Deep Controlled Learning (DCL) algorithm \citep{temizoz2025deep} emerges as a natural candidate for our application. DCL is specifically designed for environments influenced by exogenous stochastic factors, such as demand and lead time distributions, making it well suited for solving MDPs whose parameterizations encode stochastic processes (cf. \S\ref{sec:prelim}). From a methodological perspective, DCL fits the conditions for ZSG capability: it is a model-based approximate policy iteration algorithm that iteratively refines policies by distilling them into reinitialized policy networks and avoiding the need for a value network during optimization through efficient simulations. 

Since DCL does not inherently support the Super-MDP formulation for training GCAs under ZSG, we extend it by adopting the strategies outlined by \citet{Kirk2023SurveyZSGDRL}: enhancing similarity between training and testing data, and addressing differences between training and testing environments. Algorithm \ref{alg:DCL} details this adaptation, referred to as Deep Controlled Learning for Super-MDPs (Super-DCL).

\begin{algorithm}
\caption{Deep Controlled Learning for Super-Markov Decision Processes (Super-DCL)}\label{alg:DCL}
\begin{algorithmic}[1]
\State \textbf{Input}: Super-MDP model: $\hat{\mathcal{M}}_S = (\hat{\parSpace}, \mathcal{S}, \mathcal{A}, \hat{\mathcal{H}}, \mathcal{F})$, initial policy:  $\pi_{0}$, neural network structure: $\NN_{\theta}$, number of approximate policy iterations: $\bar n$, number of samples to be collected: $N$, number of threads: $w$, number of samples collected for a specific parameterization: $R$,  length of the warm-up period: $L$, simulation budget per state-action pairs: $M$, depth of the simulations: $H$, maximum number of promising actions for simulations: $P$.
\For{$i=0, 1, \dots, \bar n-1 $}
\State $\K_{i} = \{\} $, the dataset
\For{\textbf{each} $thread=1.\dots,w$}  \textbf{in parallel}  \label{alg1:paralel}
\For{$j=1, \dots, \lceil{N / w/ R}\rceil$} \label{alg1:samp1} 
\State Sample parameterization $\Pars^j \sim \hat{\mathcal{H}}$, construct $\mathcal{M}^{\Pars^j}=(\mathcal{S}, \mathcal{A}, f^{\Pars^j}, C^{\Pars^j}, \s_0^{\Pars^j})$ by $\mathcal{F}({\Pars^j})$ \label{alg1:paramgeneration}
\For{$k=1, \dots, R$} \label{alg1:r}
\State $\s_1 = (\s_{l+1} = f^{\Pars^j}(\s_{l}, \pi_i(\s_l, {\Pars^j}))$ \textbf{for} $l=0,\dots,L-1$, with $\s_0 = \s_0^{\Pars^j}$) \label{alg1:warmup}
\State Find estimated optimal action $\hat{\pi}_i^+(\s_k, {\Pars^j}) = Simulator(\mathcal{M}^{\Pars^j}, \s_{k},\pi_{i}, M, H, P)$ \label{alg1:simul}
\State Add (($\s_{k}, {\Pars^j}$), $\hat{\pi}_i^+(\s_k, {\Pars^j})$) to the data set $\K_{i}$ \label{alg1:samp2}
\State $\s_{k+1} = f^{\Pars^j}(\s_{k}, \hat{\pi}_i^+(\s_k, {\Pars^j}))$ \label{alg1:sampnew}
\EndFor
\EndFor
\EndFor
\State $\pi_{i+1} = Classifier(N_{\theta}, \K_i)$
\label{alg1:classifier}
\EndFor 
\State \textbf{Output}: $\pi_{1}, \dots, \pi_{\bar n}$ 
\end{algorithmic}
\end{algorithm}

Before presenting the details of this adaptation, we briefly summarize the algorithm’s core mechanics. Super-DCL improves an initial policy (denoted $\pi_{S,0}$ or simply $\pi_0$) through successive iterations of approximate policy improvement, framing reinforcement learning as a classification problem. At each iteration $i$, the algorithm builds a dataset $\K_i$ of parameterization–state pairs together with their estimated optimal actions. To construct $\K_i$, parameterizations are sampled from the probable parameter generator $\Pars \sim \hat{\mathcal{H}}$, and each sampled $\Pars$ induces an MDP $\mathcal{M}^\Pars = \mathcal{F}(\Pars)$ (Line~\ref{alg1:paramgeneration}). From this MDP, states are collected and a simulator estimates the optimal action $\hat{\pi}_i^+$ by evaluating candidate actions through $M$ rollouts of depth $H$ (Line~\ref{alg1:simul}), thereby providing the reinforcement signal that drives policy improvement. Each labeled pair $((\s,\Pars), \hat{\pi}_i^+(\s,\Pars))$ is then added to $\K_i$ (Line~\ref{alg1:samp2}). Once dataset $\K_i$ is complete, the policy is updated by training a neural network classifier from scratch, yielding the new policy $\pi_{i+1}$ (Line~\ref{alg1:classifier}); the network learns to map each parameterization–state pair $(\s,\Pars)$ to its estimated optimal action $\hat{\pi}_i^+(\s,\Pars)$. Repeating this process for $\bar{n}$ iterations yields a sequence of increasingly refined policies $\pi_1, \dots, \pi_{\bar{n}}$, with the final output $\pi_{\bar{n}}$ serving as an approximation of the optimal policy $\hat{\pi}_S$.

\citet{Kirk2023SurveyZSGDRL} emphasize that increasing the similarity between training and testing data can substantially reduce the generalization gap and improve performance during testing. This similarity can be achieved by ensuring that the testing environment lies within the range of environments encountered during training. Such alignment can be facilitated through two complementary strategies: \emph{extensive sampling}, where we utilize a large dataset during training, and \emph{domain randomization}, which for the DCL algorithm can be supported by (i) the introduction of a new hyperparameter $R$ that triggers resampling of new parameterizations after $R$ samples have been collected for a specific parameterization, and (ii) the design of the probable parameter generator $\hat{\mathcal H}$ to provide \emph{uniform coverage of the probable space $\hat{\mathcal P}$}, thereby exposing the agent to a wide variety of environments. Ensuring such uniformity is important: it increases the likelihood that unseen test instances fall close to trained ones, thereby strengthening ZSG. The potential benefits of these strategies are illustrated in Figure~\ref{fig:parameterspace}, which shows how broad sampling and randomized parameterization design expand the coverage of the true, unknown parameter space and thus enhance ZSG performance.

Furthermore, ZSG can be improved by training the agent to exploit differences among the environments it encounters. By incorporating parameters as features during training (Line \ref{alg1:samp2}), the GCA learns to identify and adapt to the dynamics of each task ($\mathcal{M}^\Pars$). This strategy allows the agent to effectively apply these learned distinctions during testing, thereby \emph{bridging the gap between the training and testing environments} \citep{Kirk2023SurveyZSGDRL}. Such awareness during the learning process serves as an \emph{inductive bias}, enabling the agent to accurately identify and respond to environmental characteristics at test time.

The main novelty in Super-DCL lies in moving beyond training an MDP for a fixed parameterization $\Pars^j$ to varying the parameterization using a previously defined sampling function $\hat{H}$ and reinitializing the system after $R$ sampling periods with the fixed parameterization $\Pars^j$ (Lines \ref{alg1:samp1}-\ref{alg1:sampnew}). Moreover, the policy $\pi$ not only takes states as input but also incorporates the parameterization used to generate the training samples (Line \ref{alg1:samp2}). For a comprehensive explanation of DCL, including the details of the sampling process, parallelization among threads (Line \ref{alg1:paralel}), the warm-up step for sampling a state (Line \ref{alg1:warmup}), and the $Simulator$ (Line \ref{alg1:simul}) and $Classifier$ (Line \ref{alg1:classifier}) algorithms, we direct readers to \citet{temizoz2025deep}. 

\vspace{-0.35cm}
\subsection{Estimate}\label{sec:est}
\vspace{-0.35cm}

We now assume that the policy $\hat{\pi}_S$ (trained in the Train phase) is a GCA such that $\bar{C}{\pi^*_S} + \epsilon \geq \bar{C}{\hat{\pi}_S}$. Hence, this GCA can produce effective decisions for the parameterizations generated by $\hat{\mathcal{H}}$ within the probable parameter space, $\hat{\parSpace}$, which is assumed to cover all parameterizations relevant to our decision problem ($\parSpace \subseteq \hat{\parSpace}$). However, in our original decision problem $\mathcal{D}$, the true parameterizations remain obscure, potentially affecting the GCA's effectiveness. To mitigate this, the Estimate phase of the TED framework performs real-time parameter estimation: as the DM collects observations from the environment, estimates of the true problem parameters are updated and fed to the trained GCA, while the policy itself remains fixed. As part of this phase, we analyze the impact of estimation accuracy on GCA's performance.

For a decision made in the problem $\mathcal{D}$ at time $t$, let $\Pars \in \parSpace$ denote the true parameterization of $\mathcal{D}$, and let $\hat{\Pars}_t \in \hat{\parSpace}$ represent the parameter estimates. Let $\mathcal{O}_t$ denote the observations collected up to time $t$. We define the estimator of the parameterization at time $t$ as $\hat{\Pars}_t:= \mathcal{Y}(\mathcal{O}_t)$, where $\mathcal{Y}$ maps the observations collected until time $t$ to the probable parameter space $\hat{\parSpace}$.

We now formalize how the accuracy of the estimator, measured by the similarity between the parameter estimates and the true parameterization, affects the GCA's performance. This similarity is captured by the \emph{parameterization distance}, which quantifies how two parameterizations differ in terms of their impact on the one-step dynamics of the decision-making process. While Definition \ref{def:lip} employs a generic metric $d(\cdot,\cdot)$ over the parameter space to express Lipschitz property, the parameterization distance defined below is a task-specific measure that directly captures differences in one-step costs and transition probabilities between the two induced MDPs.
\begin{definition}[Parameterization Distance]\label{def:param}
Consider a Super-MDP $\mathcal{M_S}$ and two parameterizations $\Pars_i, \Pars_j \in \hat{\parSpace}$. The parameterization distance $d_{P}(\Pars_i,\Pars_j)$ measures the largest one-step discrepancy in costs and transition probabilities between two tasks $\mathcal{M}^{\Pars_i}$ and $\mathcal{M}^{\Pars_j}$ generated by parameterizations $\Pars_i$ and $\Pars_j$. Let $C_{\max}(\Pars_i,\Pars_j) := \max_{\theta\in\{\Pars_i,\Pars_j\}}\max_{\s,a} \big|C^{\theta}(\s,a)\big|.$
Then,
\[
d_{P}(\Pars_i,\Pars_j):= \max_{\s,a}\left\{\big|C^{\Pars_i}(\s,a)-C^{\Pars_j}(\s,a)\big|+ C_{\max}(\Pars_i,\Pars_j) \;\big\|f^{\Pars_i}(\cdot\mid \s,a)-f^{\Pars_j}(\cdot\mid \s,a)\big\|_1\right\},
\]
where $\|p-q\|_1 := \sum_{x\in\mathcal{S}} |p(x)-q(x)|$ is the $\ell_1$ distance between the two transition probability distributions over next states for the same $(\s,a)$. The factor $C_{\max}(\Pars_i,\Pars_j)$ ensures the transition error term is expressed in cost units, matching the scale of the one-step cost difference.
\end{definition}

The right-hand side of the definition is bounded under the Lipschitz conditions, see Definition \ref{def:lip}. Our aim is to bound the cost associated with inaccurately estimating the true parameterization. We first define a consistent estimator, which eventually produces parameter estimates that are arbitrarily close, in the $d_P$ sense, to the true parameterization.
\begin{definition}[Consistent Estimator]\label{def:estimator}
An estimator $\mathcal{Y}:\mathcal{O}_t\to\hat{\parSpace}$ is consistent if $d_P(\Pars,\hat{\Pars}_t)\ \xrightarrow[t\to\infty]{\mathbb{P}}\ 0, \quad\text{that is,}\quad \forall \varepsilon>0:\ \lim_{t\to\infty}\Pr\big(d_P(\Pars,\hat{\Pars}_t)>\varepsilon\big)=0$.
\end{definition}

Let $\Bar{C}^{\hat{\pi}_S}_{T}(\s, \Pars)$ denote the expected per-period cost of the GCA $\hat{\pi}_S$ starting from state $\s$ under parameterization $\Pars$ over $T$ periods of time.
We can now bound the cost differences between taking an action based on the estimated parameterization $\hat{\Pars}_t$ and taking an action with knowledge of the true parameterization $\Pars$. Hence, this bound expresses the cost difference when $\hat{\Pars}_t$ is employed in a task with true parameterization $\Pars$ over $T$ periods:
\begin{thm}[Bound on Cost Difference Due to Estimation Error] \label{thm:estimate}
Fix a GCA $\hat{\pi}_S$ and a state $\s$ in a Lipschitz Super–Markov Decision Process with probable parameter space $\hat{\parSpace}$. Let $\hat{\Pars}_{1:T}$ be the parameter estimates produced by an estimator $\mathcal{Y}$ at decision epochs $t=1,\dots,T$. The per-period average cost difference satisfies
\[
|\Bar{C}^{\hat{\pi}_S}_{T}(\s, \hat{\Pars}_{1:T}) - \Bar{C}^{\hat{\pi}_S}_{T}(\s, \Pars)| \leq \frac{1}{T}\sum_{t=1}^{T}\!\Big(1+\frac{T-t}{2}\Big)\,
d_{P}(\Pars,\hat{\Pars}_t).
\]
\emph{Moreover,} if the Markov chain induced by $(f^{\Pars},\hat{\pi}_S)$ is
\emph{uniformly ergodic uniformly over} $\Pars\in\hat{\parSpace}$, then there exists
a constant $\kappa(\hat{\pi}_S)<\infty$, independent of $T$, such that
\[
\bigl|\Bar{C}^{\hat{\pi}_S}_{T}(\s, \hat{\Pars}_{1:T}) - \Bar{C}^{\hat{\pi}_S}_{T}(\s, \Pars)\bigr|
\;\le\; \frac{\kappa(\hat{\pi}_S)}{T}\sum_{t=1}^{T}\!\,
d_{P}(\Pars,\hat{\Pars}_t).
\]
\end{thm}
\begin{proof}
This proof follows from the Simulation Lemma \citep{Kearns2002NearOptimalRL} applied to the Super-MDPs; and can be found in Appendix \ref{sec:proof}.
\end{proof}
Theorem \ref{thm:estimate} quantifies how inaccuracies in parameter estimation translate into performance loss for the GCA. Specifically, it bounds the per-period average cost difference between operating with estimated parameters $\hat{\Pars}_t$ and the true parameters $\Pars$ by a weighted average of the parameterization distances $d_P(\Pars,\hat{\Pars}_t)$ over the horizon $T$. The weights $\left(1+\tfrac{T-t}{2}\right)$ capture the fact that earlier estimation errors can propagate through more future states and thus have a larger cumulative effect than later errors. This result highlights the role of a \emph{good estimator}: keeping $d_P(\Pars,\hat{\Pars}_t)$ small directly limits the worst-case performance loss. If the estimator is \emph{consistent} (Definition~\ref{def:estimator}), then $d_P(\Pars,\hat{\Pars}_t) \to 0$ in probability, tightening the finite-horizon bound as more observations are gathered. If, in addition, the chain under $\hat{\pi}_S$ is uniformly ergodic uniformly over $\Pars\in\hat{\parSpace}$ \emph{(i.e., for the fixed policy $\hat{\pi}_S$, the induced dynamics stabilize to their long-run behavior at a parameter-independent rate throughout $\hat{\parSpace}$)}, the weights become constant $\kappa(\hat{\pi}_S)$. With a consistent estimator, $\frac{1}{T}\sum_{t=1}^T d_P(\Pars,\hat{\Pars}_t)\to 0$ in probability, so the per-period average cost gap vanishes as $T\to\infty$. Consequently, the GCA can perform comparably to one with perfect parameter knowledge, achieving high performance on unseen instances without retraining.

\vspace{-0.35cm}
\subsection{Decide} \label{sec:decide}
\vspace{-0.35cm}

The \emph{Decide} phase represents decision-making steps within the TED framework. During this phase, the DM utilizes the trained GCA, $\hat{\pi}_S$, to make decisions. Specifically, for any task, the DM provides the current state of the environment, $\s$, along with the latest parameter estimates, $\hat{\Pars}$, to the GCA. The GCA then maps this pair to an appropriate action, $\hat{\pi}_S(\s, \hat{\Pars}) = a$, $a \in \mathcal{A}$. This phase capitalizes on the conditions defining the training strategies employed during the \emph{Train} phase and the accuracy of the parameter estimates refined in the \emph{Estimate} phase. If these conditions are met, the GCA can be expected to exhibit ZSG capabilities and maintain effectiveness in real-time decision-making.

However, challenges arise when the DM has no observations available to estimate the parameters, such that $\mathcal{O} = \emptyset$. In such scenarios, the GCA, which relies on parameter estimates to map states to actions, faces limitations. Simply selecting an arbitrary parameterization from the probable parameter space $\hat{\parSpace}$ may not be ideal. 

To address this, a natural approach is to employ \emph{strategies} that remain effective under a range of possible problem parameters. One such strategy is \emph{robust optimization}, which, in our context, selects a provisional parameterization from the probable parameter space that minimizes the worst-case expected cost during the initial decision periods without any observations, assuming the GCA policy is fixed. This method aims to enhance the resilience of the decision-making process by optimizing performance across a range of possible parameterizations in the absence of direct observational data. We provide the formal robust optimization model in Appendix~\ref{sec:robust-decide}.

\vspace{-0.35cm}
\section{Periodic Review Inventory Control}\label{sec:inventorymodel}
\vspace{-0.35cm}

We test our framework where $\mathcal{D}$ represents a broad class of discrete-time, periodic-review inventory problems characterized by lost sales, which exhibit the decision-making challenges defined in \S\ref{sec:prob}.

\noindent \textbf{Overview}

In general, for problems in this class, the sequence of events for any given task and time point $t$ unfolds as follows. Let $OH_t$ denote the on-hand inventory at the beginning of the period $t$. At the start of period $t$, a new shipment of inventory $q_t$ is delivered and added to $OH_t$. Then, the DM reviews the inventory and places a new order $a_t$. If the supplier allows for instant delivery, a portion of this order, $q_t^0$, may be immediately added to the on-hand inventory. The rest of the orders will be received according to a lead time distribution. The random demand $D_t$ is realized, and the DM observes the sales, which are the minimum of $D_t$ and the available inventory $OH_t+q_t+q_t^0$. The unmet demands are lost. At the end of period $t$, the DM incurs costs based on the unsatisfied demands and the on-hand inventory they have. In particular, the cost at the end of the period is: $C_t = h(OH_t+q_t+q_t^0-D_t)^+ + p(D_t-(OH_t+q_t+q_t^0))^+$, where $h$ is the holding cost of an item in inventory, and $p$ is the penalty cost for each unmet demand. However, demand may be censored in the event of stockouts, where only sales data is accessible, not the actual demand. In this case, penalty cost is not observable, and \cite{lyu2024} instead adopt maximizing the profit ($G_t$) as an objective, where $p$ is interpreted as per-unit sales revenue. Adapting this objective to our setting, we write: $G_t = p \min(D_t ,OH_t+q_t+q_t^0) - h (OH_t+q_t+q_t^0-D_t)^+$, which is applicable when the demand distribution is unknown and censored.

\noindent \textbf{Demand and supply processes}

We next briefly explain the details of the demand and supply processes that determine the amounts $D_t$, $q_t$ and $q_t^0$. The demand $D_t$ at each time step $t$ is considered a discrete random variable within the bounds of $[0, D_{\max}]$ and follows a structured pattern. We model this demand as occurring in cycles of constant length $K$, where $K$ may span from $1$ to $K_{max}$. The demands are independent across the periods but not identically distributed within a cycle, allowing for a variety of distribution patterns. Mathematically, the demands $D_t$ can be grouped into $K$ subsets according to the remainder when $t$ is divided by $K$ (i.e., $t \mod K$), with each subset of demands being generated independently from its specific distribution \citep[see][for a similar demand process analysis]{Gong2023}. 

The supply process in our inventory system, determining the amounts $q_t$ and $q_t^0$, is influenced by stochastic lead times. To model this process, we reference two prevalent approaches in the literature: orders can cross, where an order that is placed after another order may be received before that order \citep[order crossover, see][]{Stolyar2022}, and orders cannot cross (see \citet{kaplan1970}). In the former case, the lead times are i.i.d., while in the latter case orders are received in the same sequence in which they were placed. We incorporate both cases as they have practical relevance; see \citet{andaz2024learning} for order crossing, and \citet{bai2023asymptotic} for the latter case and a general analysis on stochastic lead times.

\noindent \textbf{Super-MDP formulation}

Our problem class includes instances with a wide range of cost parameters, demand distributions (including cyclic demand), and lead time distributions (with or without order crossover). When constructing a Super-MDP for this problem during the Train phase, we focus on strategies to maintain the Lipschitz conditions of the Super-MDP (see Definition \ref{def:lip}). The probable parameter space and the associated parameter generation function are created accordingly. We define bounds for the cost parameter $p$ and parameters related to demand and lead time distributions, and we let demand and lead time distributions arise from two-moment fits \cite[following][]{adan1995} and other techniques. For details, we refer to Appendix \ref{sec:smdplostsales}, which contains a detailed description of the construction of the Super-MDP as well as the subsequent training of the GCA. 

In the Estimate phase, Theorem \ref{thm:estimate} encourages the use of any effective technique to accurately estimate the true parameters. We assume that the holding cost $h$ and penalty cost $p$ are directly observable, while the demand distribution $\boldsymbol\zeta$ and lead time distribution $\boldsymbol\tau$ are unknown and should be estimated from historical data. When demands are censored, where only sales data is available, true demand can be systematically underestimated, potentially causing continuous understocking \citep{Huh2009analysis}. To mitigate this, we apply the Kaplan–Meier estimator \citep{kaplan1958nonparametric} to estimate the demand distribution under censored observations; similar adaptations appear in \citet{huh2011adaptive} and \citet{lyu2024}.

Lead time distributions are estimated empirically via relative frequencies: the probability of each lead time value is obtained by normalizing its observed frequency by the total number of observations. Further details are provided in Appendix~\ref{sec:estimatelostsales}, which describes the \emph{Estimate} and \emph{Decide} phases in our setting, along with our strategy for selecting a parameterization when no observations are yet available for the unknown parameters and the proofs regarding the consistency of these estimators in the sense of Definition \ref{def:estimator}. 

\vspace{-0.35cm}
\section{Numerical Experiments}\label{sec:numex}
\vspace{-0.35cm}

We implement the Super-DCL algorithm to train the GCA for addressing the periodic review inventory control problem discussed in \S\ref{sec:inventorymodel}. The algorithm is programmed in C++20 and executed on five AMD EPYC 9654 processors, each with 192 hardware threads. The GCA, herein referred to as the \emph{Generally Capable Lost Sales Network - GC-LSN}, underwent a training duration of $\sim16$ hours. Information on the bounds for the probable parameter space of the constructed Super-MDP and hyperparameters used in the Super-DCL algorithm can be found in Appendix \ref{sec:exsetup}. The features used in GC-LSN are described in Appendix \ref{sec:features}, and consist of problem parameter estimates as well as a state space representation. The trained GC-LSN represents a single generally applicable policy that has been employed for all numerical experiments reported in this section.  

For the numerical results, we organize our problem instances into groups referred to as \emph{Cases}. Each \emph{Case} comprises a collection of instances that share similar inventory challenges but differ in their parameter settings. Specifically: \emph{Case 1} - iid demand with deterministic lead times, \emph{Case 2} - cyclic demand patterns with deterministic lead times, and \emph{Case 3} - stochastic lead times.

\emph{Case 1} includes 320 problem instances where demand in each period is iid, and lead times are deterministic. These instances are structured using a full factorial design with the following parameter ranges: mean demand $\mu \in \{3.0, 5.0, 7.0, 10.0\}$, penalty cost $p \in \{9.0, 39.0, 69.0, 99.0\}$, deterministic lead time $\in \{2, 4, 6, 8, 10\}$. For each mean demand value, the standard deviation is selected such that the demand distribution follows one of the following: binomial distribution (very low variance), Poisson distribution (low variance), negative binomial distribution (high variance) and geometric distribution (very high variance). In Appendix \ref{sec:instances}, we provide an “empirical” version of Figure \ref{fig:parameterspace} concerning some instances of Case 1.

\emph{Case 2} encompasses 243 problem instances where demand follows cyclic patterns, and lead times remain deterministic. The parameter ranges for these instances are: $K \in \{ 3, 5, 7\}$, $p \in \{9.0, 39.0, 69.0\}$, lead time $\in \{3, 6, 9\}$. Details regarding the specific demand distributions within each cycle can be found in Appendix \ref{sec:instances}.

\emph{Case 3} consists of 240 problem instances with stochastic lead times. These instances are further divided into two conditions: order crossover - in half of the instances, orders can cross; sequential orders - in the remaining instances, orders are received in the exact sequence they were placed. For each condition, we adopt 10 distinct lead time distributions. The ranges of other parameters are: $K \in \{ 1, 3, 5, 7\}$, $p \in \{9.0, 39.0, 69.0\}$. For details regarding the lead time and demand distributions, we refer to Appendix \ref{sec:instances}.

We report the numerical results in two parts. We first investigate the performance of GC-LSN compared to classical benchmarks when the problem parameters are considered \emph{known}. The main purpose of these experiments is to validate that GC-LSN is a competitive policy under a wide range of inventory challenges. Then, we move to the case where the parameters of the problem instances, demand and/or lead time distributions are unknown, must be estimated, and hence GC-LSN is fed with \emph{estimated} parameters. Moreover, we also quantify the performance and operational trade-off where we face a choice between training and deploying a GCA with ZSG ability and training \emph{individual agents} tailored to fixed, known parameters. We report the corresponding results in Appendix \ref{sec:per-instance}.

\vspace{-0.35cm}
\subsection{Benchmarking performance when problem parameters are known} \label{sec:benchmarkingknownparameters}
\vspace{-0.35cm}

When problem parameters are known, we bypass the Estimate phase of the TED framework and can directly apply GC-LSN to each instance. We adopt the base-stock policy (BSP) and capped base-stock policy (C-BSP) \citep{Xin2021} as our benchmarks. We conduct a simulation-based optimization of their parameters (i.e. base-stock level) for each instance, assuming full knowledge of the true problem parameters (demand and lead time distributions), thereby placing the benchmarks in the same informational position as GC-LSN for this set of experiments. We also note that there is no tailored approach capable of handling the inventory challenges present in \emph{Case 2} and \emph{Case 3}, where effective heuristics for systems with cyclic demands are lacking \citep{Gong2023}. 

When comparing GC-LSN against the benchmarks, we report results in terms of \emph{relative cost gap}, defined as: $Gap(\hat{\pi}_S, \pi_B)_{C,T} = \frac{\Bar{C}^{\hat{\pi}_S}_{T} - \Bar{C}^{\pi_{B}}_T} {\Bar{C}^{\pi_B}_T}$. Here, $\Bar{C}^{\hat{\pi}_S}_T$ denotes the average per period cost of GC-LSN over $T$ periods, and $\Bar{C}^{\pi_B}_T$ denotes the average per period cost of the benchmark policy $\pi_B$ over the same $T$ periods. A negative gap indicates that GC-LSN outperforms the benchmark. We obtain unbiased estimators of the average costs for both GC-LSN and the benchmarks per period through simulations. Each evaluation consists of 1000 runs, with each run spanning $T=5000$ periods and initiating after a warm-up period of 100 periods. The results are statistically significant, with the half-width of a 95\% confidence interval being less than 1\% of the corresponding cost value.

Table \ref{tab:fullinfo} presents the relative cost gap of GC-LSN compared to the BSP and the C-BSP. For Case 1, we find that GC-LSN consistently outperforms both benchmark policies across varying penalty costs ($p$). We observe that the cost gap increases, approaching zero, with higher penalty costs, which is expected since both benchmark policies are asymptotically optimal as the penalty cost grows. For Case 2, GC-LSN consistently outperforms both policies by a larger margin, with the difference being more pronounced when demands exhibit cyclic patterns, probably because for those cases the benchmarks are less effective. For Case 3, GC-LSN again demonstrates superior performance by achieving negative regret compared to the benchmark policies. This highlights GC-LSN's ability to effectively utilize complete information in managing inventory under stochastic lead times. These results validate GC-LSN as a GCA, effectively managing inventory under varied inventory challenges. 

\begin{table}[ht]
   \centering
   \footnotesize
\begin{tabular}{r|c|c|c|c|c|c|c|c|c|}
    \cmidrule{2-10}
    \multicolumn{1}{c}{\multirow{1}[1]{*}{}} &\multicolumn{4}{|c|}{Case 1} & \multicolumn{3}{c|}{Case 2} & \multicolumn{2}{c|}{Case 3} \\
    \cmidrule{2-10}
    \multicolumn{1}{c|}{\multirow{1}[1]{*}{}} & \multicolumn{1}{c}{$p=9$} & \multicolumn{1}{c}{$p=39$} & \multicolumn{1}{c}{$p=69$} & \multicolumn{1}{c|}{$p=99$} & \multicolumn{1}{c}{$K=3$} & \multicolumn{1}{c}{$K=5$} & \multicolumn{1}{c|}{$K=7$} & \multicolumn{1}{c}{$l=0$} & \multicolumn{1}{c|}{$l=1$} \\
    \midrule
    BSP & \multicolumn{1}{c}{$-7.4\%$} & \multicolumn{1}{c}{$-3.3\%$} & \multicolumn{1}{c}{$-2.3\%$} & \multicolumn{1}{c|}{$-1.7\%$} & \multicolumn{1}{c}{$-8.8\%$} & \multicolumn{1}{c}{$-7.2\%$} & \multicolumn{1}{c|}{$-7.8\%$} & \multicolumn{1}{c}{$-2.6\%$} & \multicolumn{1}{c|}{$-1.5\%$} \\
    C-BSP & \multicolumn{1}{c}{$-0.9\%$} & \multicolumn{1}{c}{$-0.9\%$} & \multicolumn{1}{c}{$-0.7\%$} & \multicolumn{1}{c|}{$-0.5\%$} & \multicolumn{1}{c}{$-4.1\%$} & \multicolumn{1}{c}{$-4.1\%$} & \multicolumn{1}{c|}{$-4.4\%$} & \multicolumn{1}{c}{$-1.3\%$} & \multicolumn{1}{c|}{$-1.3\%$} \\
    \bottomrule    
\end{tabular}
    \caption{Relative cost gap of GC-LSN - the lower the better}
    \label{tab:fullinfo}
\end{table}

\vspace{-0.35cm}
\subsection{Benchmarking performance when problem parameters are unknown}
\vspace{-0.35cm}

In \S\ref{sec:benchmarkingknownparameters}, we demonstrated GC-LSN's performance for complex decision-making problems where the problem parameters are known. We now turn to the more challenging setting in which these parameters are \emph{unknown} and must be estimated. Hence, we deploy GC-LSN with \emph{estimated} parameters (using the methods described in \S\ref{sec:inventorymodel}), denoting this policy as GC-LSN-E. This is the same policy as GC-LSN, differing only in that it receives parameter estimates rather than true parameter values.

When reporting the results, we obtain unbiased estimators of the average profits per period (when the demand distribution is unknown, because the DM can only observe the sales data, cf. \S\ref{sec:inventorymodel}) and average costs per period (when only the lead time distribution is unknown) of policies through simulations. Each evaluation is comprised of $1000$ runs, with each run spanning $200, 500, 1000$ or $2000$ periods - without any warm-up period, to assess how well our agent adapts under unknown distribution conditions as it gathers more observations.

In this section, we consider online learning algorithms, \textbf{clairvoyant} BSP (Cv.\ BSP),  \textbf{clairvoyant} C-BSP (Cv.\ C-BSP), and GC-LSN with full parameter knowledge as benchmarks. By "clairvoyant," we mean that the parameters of these policies (i.e., base-stock level) are optimized assuming full knowledge of the true problem parameters, similar to the experiments in Section \ref{sec:benchmarkingknownparameters}. When comparing GC-LSN-E against these benchmarks, we report relative improvements in terms of costs and profits. When the demand distribution is unknown, we use relative profit gap: $Gap(\hat{\pi}_S^E, \pi_B)_{G,T} = \frac{\Bar{G}^{\pi_{B}}_{T} - \Bar{G}^{\hat{\pi}_S^E}_T} {\Bar{G}^{\pi_B}_T}$. Here, $\Bar{G}^{\hat{\pi}_S^E}_T$ denotes the average profit of GC-LSN-E over $T$ periods, and $\Bar{G}^{\pi_B}_T$ denotes the average profit of the benchmark policy $\pi_B$ over the same $T$ periods. When only the lead time distribution is unknown, we utilize relative cost gap: $Gap(\hat{\pi}_S^E, \pi_B)_{C,T} = \frac{\Bar{C}^{\hat{\pi}_S^E}_{T} - \Bar{C}^{\pi_{B}}_T} {\Bar{C}^{\pi_B}_T}$.

\vspace{-0.35cm}
\subsubsection{Benchmarking where online learning algorithms are applicable}
\vspace{-0.35cm}

GC-LSN is very broadly applicable, and our rationale for selecting this particular setting is that it is the focus of several studies that focus on online learning for lost sales inventory control; see Section~\ref{sec:onlinelearninglit}. Such studies provide us with policies that are designed for decision-making in online settings where demand distributions are unknown, and that hence are suitable for benchmarking the full TED framework. 

To this end, we adopt the test bed and experiment configuration proposed by \citet{lyu2024}, and consider the six online learning algorithms tested in that study as a benchmark (this includes the state-of-the-art methods proposed by \citet{lyu2024} as well as \citet{Agrawal2022regret} and \citet{zhang2020}). We report the results of the best performing benchmark for each instance; the acronym $\pi_B$ is used to refer to this benchmark. 

Table \ref{tab:comparisonresults} displays the average profits of GC-LSN-E and the best benchmark for each instance $\pi_B$ across various horizon lengths. The results indicate that GC-LSN-E consistently outperforms the best existing learning algorithms for lost sales in nearly all instances, without necessitating additional training or simulations. As expected, the performance of GC-LSN-E improves with longer horizons. This improvement is attributed to the enhanced accuracy of parameter estimates, which benefit from an increased number of observations. The Kaplan-Meier estimator is apparently well-suited to this task when employed alongside GC-LSN-E. 

\begin{table}[ht]
  \centering
  \footnotesize
\begin{tabular}{|r|l|l|cccc|cccc|}
\cmidrule{4-11}    \multicolumn{3}{c|}{\multirow{2}[4]{*}{}} & \multicolumn{8}{c|}{Poisson distribution with mean 10} \\
\cmidrule{4-11}    \multicolumn{3}{c|}{} & \multicolumn{4}{c}{$p=5.0$} & \multicolumn{4}{|c|}{$p=10.0$} \\
    \midrule
    \multicolumn{1}{|r|}{Lead time} & Policy & Period  &  200     & 500     & 1000 & 2000 &  200     & 500     & 1000     & 2000   \\
    \midrule
    \multirow{2}[2]{*}{$1$} & GC-LSN-E & Average Profit&
    \textbf{42.7} & \textbf{43.1} & \textbf{43.2} & \multicolumn{1}{r|}{\textbf{43.3}} &
    \textbf{91.1} & \textbf{91.6} & \textbf{91.8} & \multicolumn{1}{r|}{\textbf{91.9}}  \\
        & $\pi_B$   & Average Profit&
    41.9 & 42.6 & 42.9 & \multicolumn{1}{r|}{43.1} &
    89.4 & 90.3 & 90.8 & \multicolumn{1}{r|}{91.0}  \\
    \midrule
    \multirow{2}[2]{*}{$3$} & GC-LSN-E & Average Profit&
    \textbf{41.6} & \textbf{42.1} & \textbf{42.3} & \multicolumn{1}{r|}{\textbf{42.5}} &
    \textbf{88.4} & \textbf{89.5} & \textbf{89.9} & \multicolumn{1}{r|}{\textbf{90.1}}  \\
        & $\pi_B$   & Average Profit&
    40.2 & 41.2 & 41.6 & \multicolumn{1}{r|}{41.8} &
    86.5 & 87.8 & 88.4 & \multicolumn{1}{r|}{88.8}  \\
    \midrule
        \multirow{2}[2]{*}{$5$} & GC-LSN-E & Average Profit&
    \textbf{40.7} & \textbf{41.5} & \textbf{41.8} & \multicolumn{1}{r|}{\textbf{42.0}} &
    \textbf{86.6} & \textbf{88.2} & \textbf{88.8} & \multicolumn{1}{r|}{\textbf{89.2}}  \\
        & $\pi_B$   & Average Profit&
    38.6 & 40.2 & 40.8 & \multicolumn{1}{r|}{41.1} &
    84.2 & 86.3 & 87.2 & \multicolumn{1}{r|}{87.9}  \\
    \midrule
        \multirow{2}[2]{*}{$7$} & GC-LSN-E & Average Profit&
    \textbf{39.9} & \textbf{41.1} & \textbf{41.6} & \multicolumn{1}{r|}{\textbf{41.8}} &
    \textbf{85.0} & \textbf{87.3} & \textbf{88.2} & \multicolumn{1}{r|}{\textbf{88.6}}  \\
        & $\pi_B$   & Average Profit&
    36.7 & 39.1 & 40.0 & \multicolumn{1}{r|}{40.5} &
    82.5 & 85.5 & 86.7 & \multicolumn{1}{r|}{87.3}  \\
    \midrule
\multicolumn{3}{c|}{\multirow{2}[4]{*}{}} & \multicolumn{8}{c|}{Geometric distribution with mean 10} \\
\cmidrule{4-11}    \multicolumn{3}{c|}{} & \multicolumn{4}{c}{$p=5.0$} & \multicolumn{4}{|c|}{$p=10.0$} \\
    \midrule
    \multicolumn{1}{|r|}{Lead time} & Policy & Period  &  200     & 500     & 1000 & 2000 &  200     & 500     & 1000     & 2000   \\
    \midrule
    \multirow{2}[2]{*}{$1$} & GC-LSN-E & Average Profit&
    \textbf{27.2} & \textbf{27.5} & \textbf{27.7}& \multicolumn{1}{r|}{\textbf{27.8}} &
    68.1 & 69.0 & 69.6 & \multicolumn{1}{r|}{69.9}  \\
        & $\pi_B$   & Average Profit&
    26.7 & 27.2 & 27.3 & \multicolumn{1}{r|}{27.5} &
    \textbf{69.4} & \textbf{69.7} & \textbf{69.9} & \multicolumn{1}{r|}{69.9}  \\
    \midrule
    \multirow{2}[2]{*}{$3$} & GC-LSN-E & Average Profit&
    \textbf{25.5} & \textbf{25.8} & \textbf{26.0} & \multicolumn{1}{r|}{\textbf{26.1}} &
    \textbf{64.4} & \textbf{65.6} & \textbf{66.2} & \multicolumn{1}{r|}{\textbf{66.6}}  \\
        & $\pi_B$   & Average Profit&
    24.1 & 24.9 & 25.2 & \multicolumn{1}{r|}{25.5} &
    64.1 & 65.5 & 65.1 & \multicolumn{1}{r|}{65.4}  \\
    \midrule
        \multirow{2}[2]{*}{$5$} & GC-LSN-E & Average Profit&
    \textbf{24.7} & \textbf{25.3} & \textbf{25.6} & \multicolumn{1}{r|}{\textbf{25.8}} &
    \textbf{62.4} & \textbf{63.8} & \textbf{64.6} & \multicolumn{1}{r|}{\textbf{65.0}}  \\
        & $\pi_B$   & Average Profit&
    22.7 & 23.8 & 24.3 & \multicolumn{1}{r|}{24.6} &
    61.0 & 62.6 & 63.1 & \multicolumn{1}{r|}{63.5}  \\
    \midrule
        \multirow{2}[2]{*}{$7$} & GC-LSN-E & Average Profit&
    \textbf{23.8} & \textbf{24.7} & \textbf{25.0} & \multicolumn{1}{r|}{\textbf{25.2}} &
    \textbf{60.4} & \textbf{62.4} & \textbf{63.4} & \multicolumn{1}{r|}{\textbf{63.9}}  \\
        & $\pi_B$   & Average Profit&
    21.0 & 22.7 & 23.7 & \multicolumn{1}{r|}{24.2} &
    58.2 & 60.3 & 61.2 & \multicolumn{1}{r|}{61.9}  \\
    \bottomrule    
    \end{tabular}%
       \caption{Average profit of GC-LSN-E compared to the best reported learning algorithm result, $\pi_B$, in the test instances taken from \citet{lyu2024}. We highlight the largest profit for each problem instance in \textbf{bold}.}
  \label{tab:comparisonresults}%
\end{table}%

To gain insights into how quickly GC-LSN-E can produce effective actions through demand distribution estimation, we compare it against benchmarks that have full information about the demand distribution, Cv.\ BSP. Cv.\ C-BSP, and GC-LSN. In Table \ref{tab:regretresults}, we report the relative profit gap of GC-LSN-E against these benchmarks using the same test bed employed by \citet{lyu2024}.

\begin{table}
  \centering
  \footnotesize
\begin{tabular}{|r|l|l|rrrrr|}
    \midrule
    \multicolumn{1}{|r|}{Policy} & Benchmark & Period  &  200     & 500     & 1000 & 2000  & 5000 \\
    \midrule
    \multirow{3}[2]{*}{GC-LSN-E} & Cv.\ BSP & Gap $\%$ &
    $-0.9\%$ & $-1.3\%$ & $-1.5\%$ &$-1.6\%$ & \multicolumn{1}{r|}{$-1.7\%$} \\
        & Cv.\ C-BSP  & Gap $\%$ &
    $2.5\%$ & $1.6\%$ & $1.2\%$ &$1.0\%$ & \multicolumn{1}{r|}{$0.8\%$} \\
       & GC-LSN  & Gap $\%$ &
    $2.8\%$ & $1.9\%$ & $1.5\%$ &$1.3\%$ & \multicolumn{1}{r|}{$1.1\%$} \\
    \bottomrule    
    \end{tabular}%
       \caption{Relative profit gap—lower is better—of GC-LSN-E vs.\ Cv.\ BSP, Cv.\ C-BSP and GC-LSN in \citet{lyu2024} instances. The benchmark policies are optimized under full information; GC-LSN-E uses parameter estimates. }
  \label{tab:regretresults}%
\end{table}%

One would expect this gap to be non-negative, as the benchmark policies operate with full information. Yet, Table~\ref{tab:regretresults} shows that GC-LSN-E can outperform Cv.\ BSP within as few as 200 periods, despite this informational disadvantage. Unlike typical online learning algorithms that aim to match BSP performance asymptotically, GC-LSN-E surpasses it rapidly. The regret against Cv.\ C-BSP starts at $2.5\%$ and steadily declines with horizon length, reflecting the success of the Kaplan–Meier estimator in estimating the true parameters. When compared to GC-LSN, the performance gap also decreases quickly and remains well below the worst-case bound in Theorem~\ref{thm:estimate}.

\vspace{-0.35cm}
\subsubsection{Benchmarking where there is no online learning algorithm available}
\vspace{-0.35cm}

Lastly, we evaluate the performance of GC-LSN-E on our extensive test beds, \emph{Case 1}, \emph{Case 2}, and \emph{Case 3}, when the problem parameters are unknown. It is important to note that there are no well-performing online learning algorithms available for \emph{Case 2} and \emph{Case 3}. For instance, \citet{Gong2023} investigate online learning for cyclic demands in lost sales inventory control with zero lead times, while \citet{ChenLyuYuanZhou2023} explore online learning for stochastic lead times, but studies the constant order policy, which performs poorly when the long-run average cost is sensitive to lead times.

For Case 1 and Case 2, we assume unknown demand distribution while for Case 3 we first consider unknown demand distribution with known lead time distribution, then unknown lead time distribution with known demand distribution, and lastly unknown demand and lead time distributions. This comprehensive approach allows us to assess the performance of GC-LSN-E across varied inventory challenges under different conditions of parameter uncertainty.

Table \ref{tab:regretresultscensoredcases} presents the relative profit gap and relative cost gap of GC-LSN-E compared to the Cv.\ BSP, Cv.\ C-BSP, and GC-LSN. We observe that the performance of GC-LSN-E improves as the number of periods increases, benefiting from more accurate estimates of demand and lead time distributions over time. In several instances, GC-LSN-E outperforms both Cv.\ BSP and Cv.\ C-BSP. Notably, GC-LSN-E maintains strong performance even when both demand and lead time distributions are unknown. These results underscore the efficacy of GC-LSN-E as a GCA under ZSG conditions. This capability is particularly evident in experimental setups where no other solution approaches are available, highlighting GC-LSN-E's robustness and adaptability in managing inventory challenges with incomplete information.

\begin{table}[ht]
\centering
\footnotesize
\begingroup
\setlength{\tabcolsep}{3.5pt}   
\renewcommand{\arraystretch}{0.95} 

\begin{adjustbox}{max width=\textwidth}
\begin{tabular}{c|c|rrrr|rrrr|rrrr|}
\cmidrule{3-14}
\multicolumn{1}{c}{} &
\multicolumn{1}{c|}{} &
\multicolumn{4}{c}{Cv.\ BSP} &
\multicolumn{4}{|c|}{Cv.\ C-BSP} &
\multicolumn{4}{c|}{GC-LSN} \\
\cmidrule{3-14}
\multicolumn{1}{c}{} &
\multicolumn{1}{c|}{} &
\multicolumn{1}{c}{$200$} &
\multicolumn{1}{c}{$500$} &
\multicolumn{1}{c}{$1000$} &
\multicolumn{1}{c|}{$2000$} &
\multicolumn{1}{c}{$200$} &
\multicolumn{1}{c}{$500$} &
\multicolumn{1}{c}{$1000$} &
\multicolumn{1}{c|}{$2000$} &
\multicolumn{1}{c}{$200$} &
\multicolumn{1}{c}{$500$} &
\multicolumn{1}{c}{$1000$} &
\multicolumn{1}{c|}{$2000$} \\
\midrule
\multirow{4}{*}{\shortstack{Case 1\\Unknown demand}} &
$p=9$  & $3.6$ & $0.1$ & $-1.1$ & $-1.7$ & $7.0$ & $2.9$ & $1.5$ & $0.8$ & $7.3$ & $3.2$ & $1.8$ & $1.1$ \\
& $p=39$ & $1.2$ & $0.4$ & $0.0$ & $-0.2$ & $1.8$ & $0.7$ & $0.4$ & $0.1$ & $2.0$ & $0.9$ & $0.5$ & $0.3$ \\
& $p=69$ & $1.0$ & $0.4$ & $0.1$ & $0.0$ & $1.2$ & $0.5$ & $0.3$ & $0.1$ & $1.3$ & $0.6$ & $0.3$ & $0.2$ \\
& $p=99$ & $0.8$ & $0.3$ & $0.1$ & $0.0$ & $1.0$ & $0.4$ & $0.2$ & $0.1$ & $1.0$ & $0.5$ & $0.3$ & $0.1$ \\
\midrule
\multirow{3}{*}{\shortstack{Case 2\\Unknown demand}} &
$K=3$  & $1.6$ & $-0.4$ & $-1.1$ & $-1.5$ & $3.5$ & $1.2$ & $0.4$ & $0.0$ & $4.1$ & $1.9$ & $1.1$ & $0.6$ \\
& $K=5$ & $2.2$ & $0.3$ & $-0.4$ & $-0.8$ & $3.5$ & $1.3$ & $0.5$ & $0.1$ & $4.2$ & $1.9$ & $1.1$ & $0.7$ \\
& $K=7$ & $2.7$ & $0.4$ & $-0.4$ & $-0.9$ & $4.2$ & $1.6$ & $0.7$ & $0.2$ & $5.0$ & $2.4$ & $1.4$ & $0.9$ \\
\midrule
\multirow{2}{*}{\shortstack{Case 3\\Unknown demand}} &
$l=0$  & $0.7$ & $-0.1$ & $-0.5$ & $-0.7$ & $1.4$ & $0.5$ & $0.1$ & $-0.1$ & $1.9$ & $0.9$ & $0.5$ & $0.3$ \\
& $l=1$ & $2.3$ & $1.0$ & $0.5$ & $0.2$ & $2.1$ & $1.0$ & $0.5$ & $0.3$ & $2.3$ & $1.1$ & $0.7$ & $0.5$ \\
\cmidrule{2-14}
\multirow{2}{*}{\shortstack{Case 3\\Unknown lead time}} &
$l=0$  & $-1.6$ & $-2.1$ & $-2.3$ & $-2.4$ & $-0.6$ & $-0.9$ & $-1.1$ & $-1.1$ & $1.4$ & $0.7$ & $0.4$ & $0.2$ \\
& $l=1$ & $3.5$ & $0.9$ & $-0.2$ & $-0.8$ & $0.4$ & $-0.5$ & $-0.9$ & $-1.1$ & $2.0$ & $1.0$ & $0.5$ & $0.2$ \\
\cmidrule{2-14}
\multirow{2}{*}{\shortstack{Case 3\\Unknown demand--lead time}} &
$l=0$  & $0.8$ & $-0.1$ & $-0.4$ & $-0.6$ & $1.5$ & $0.5$ & $0.1$ & $-0.1$ & $2.0$ & $0.9$ & $0.5$ & $0.3$ \\
& $l=1$ & $3.0$ & $1.3$ & $0.7$ & $0.3$ & $2.7$ & $1.2$ & $0.7$ & $0.4$ & $3.0$ & $1.4$ & $0.9$ & $0.6$ \\
\bottomrule
\end{tabular}
\end{adjustbox}
\endgroup

\caption{Relative profit gap (unknown demands) and relative cost gap (unknown lead time, known demands)—lower is better; \% omitted—of GC-LSN-E vs.\ Cv.\ BSP, Cv.\ C-BSP and GC-LSN. The benchmark policies are optimized under full information; GC-LSN-E uses parameter estimates.}
\label{tab:regretresultscensoredcases}
\end{table}

Moreover, the results also show that the performance gap between GC-LSN-E and GC-LSN consistently narrows as the horizon grows. This shrinkage is far more pronounced than the worst-case bound given in Theorem \ref{thm:estimate}, suggesting that the parameterization distance $d_P(\Pars,\hat{\Pars}_t)$ becomes small within a relatively short number of periods. This pattern also matches Corollary~\ref{cor:dp-consistency} in Appendix \ref{sec:estimatelostsales}: our demand and lead-time estimators are consistent, so the learned parameters approach the true ones as data accumulate and the policy’s performance converges to the clairvoyant benchmark (GC-LSN). The sustained decrease of the gap over longer horizons empirically implies a uniformly ergodic underlying chain under the deployed policy. While establishing such uniform ergodicity analytically is highly involved and beyond our scope, the experiments provide strong evidence in its favor.

\vspace{-0.35cm}
\section{Discussion and Ablation Experiments}\label{sec:discussion}
\vspace{-0.35cm}

GC-LSN provides exceptional performance in numerical settings, outperforming state-of-the-art benchmarks across diverse inventory challenges as a \emph{single} policy. We now discuss how this level of performance is achieved at the action level, focusing on how GC-LSN makes ordering decisions and how these differ from BSP and C-BSP. We also analyze the sensitivity of constructing the probable parameter space $\hat{\parSpace}$, particularly with respect to its bounds, and examine the robustness of GC-LSN when deployed in instances that slightly deviate from $\hat{\parSpace}$, i.e., when $\hat{\parSpace}$ is not a strict superset of $\parSpace$.

The success of GC-LSN across a variety of parameter settings lies in the training process. As emphasized in \S\ref{sec:prob} and \S\ref{sec:train}, GC-LSN is trained under the Super-MDP formulation using the Super-DCL algorithm, where a neural network agent learns a mapping from state–parameterization pairs to estimated optimal actions. While this mapping is a black box, it provides a powerful approximation mechanism that can capture complex nonlinear dependencies between states, parameters, and actions. This is particularly important in inventory problems such as lost sales inventory control, where deriving closed-form optimal policies is notoriously difficult \citep{zipkin2008old}.

To examine the underlying decision behavior of GC-LSN, we first compare the service levels (the fraction of demand met from stock) achieved by GC-LSN, BSP, and C-BSP across the three test cases in \S\ref{sec:benchmarkingknownparameters}. Interestingly, all three policies achieve almost identical service levels, matching within one percent across all cases (see Appendix~\ref{sec:ablation}). This demonstrates that GC-LSN’s advantage does not stem from trivially holding more inventory to reduce stockouts. Instead, its performance gains arise from a more efficient balance between holding and penalty costs. This suggests that GC-LSN avoids excessive overstocking while still preventing costly stockouts.

To gain further insights, we examine how policies adjust their actions when a parameter changes while the state remains fixed. The left plot of Figure~\ref{fig:actions} illustrates this experiment, where we vary the penalty cost $p \in \{9.0, 39.0, 69.0, 99.0\}$. While the actions scale with $p$ as expected, we observe three distinct behaviors: BSP exhibits more greedy adjustments, and C-BSP is more conservative because the orders are capped.  Notably, GC-LSN demonstrates a smooth nonlinear adjustment pattern, highlighting a black box behavior that is not apparent from average cost comparisons alone.

\begin{figure}[ht]
\centering
\includegraphics[width=1.0\textwidth]{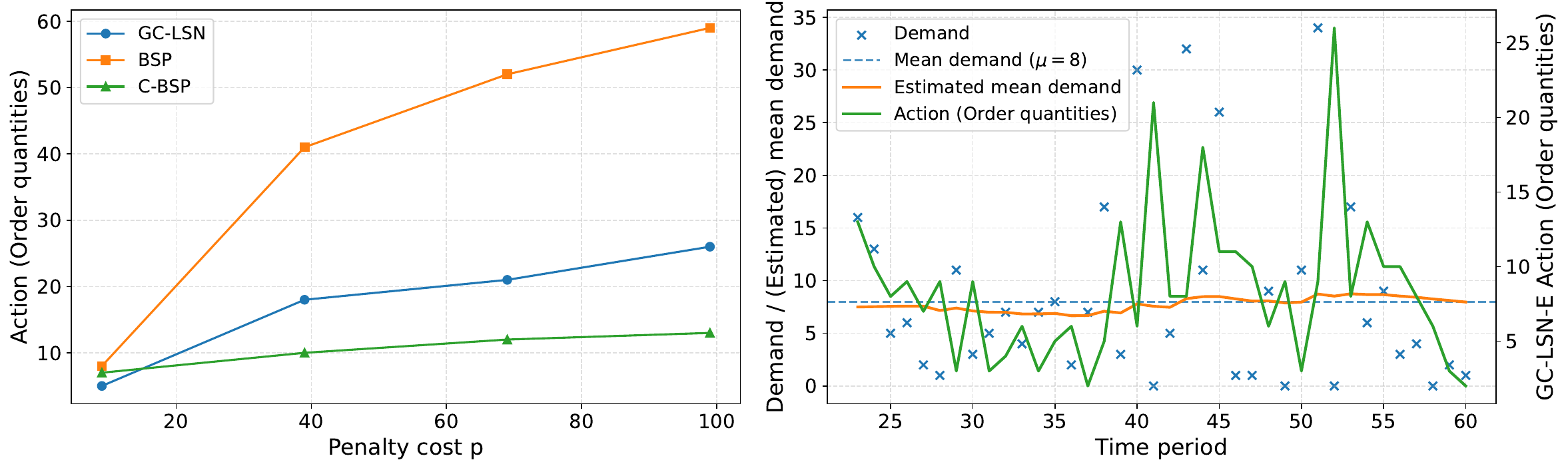}
\caption{\textbf{Left:} Order quantities placed by GC-LSN, BSP, and C-BSP for a fixed state (on hand inventory and each pipeline have 8 units) when the penalty cost $p$ increases. Demand ($K=1$) is geometric distributed with mean $8$ and lead time is 8. \textbf{Right:} Order quantities taken by GC-LSN-E when supplied with estimated demand distributions. Demand ($K=1$) is geometric distributed with mean $8$, lead time is 8, $p=69.0$.}
\label{fig:actions}
\end{figure}

Moreover, we analyze how GC-LSN-E takes actions when supplied with updated parameter estimates under varying demand realizations. The right plot of Figure~\ref{fig:actions} illustrates this experiment. We observe that the Kaplan–Meier estimator already produces an estimated mean demand close to the true mean after roughly 20 periods, even in a high-variance setting (geometric demand). While the order quantities of GC-LSN-E fluctuate with observed demands, extreme realizations do not lead to equally extreme actions: in periods of unusually high demand, GC-LSN-E typically orders less than the observed demand, whereas in periods of unusually low demand, it orders more. By contrast, BSP directly mirrors the observed demand, and C-BSP enforces a cap on orders. These differences indicate that GC-LSN-E avoids both excessive variability and over-conservative adjustments.

The results reported so far for GC-LSN concern instances (Case 1, Case 2, Case 3) contained in the probable parameter space $\hat{\parSpace}$ of the formulated Super-MDP. By Assumption~\ref{as:paramspace}, we have $\hat{\parSpace} \supseteq \parSpace$. In practice, while such bounds can be chosen using domain knowledge, the sensitivity of the trained GCA to these bounds is an important design choice. A wider $\hat{\parSpace}$ increases the chance of covering the true parameterization but requires more extensive sampling and domain randomization (\S\ref{sec:train}), leading to longer training times. Conversely, a smaller $\hat{\parSpace}$ reduces training effort but risks missing instances from the true parameter space.

Appendix~\ref{sec:exsetup} provides the detailed parameter ranges used in our experiments. To bound the penalty cost $p$ and mean demand $\mu$, we set $p_{\min}=2.0$, $p_{\max}=100.0$, $\mu_{\min}=2.0$, and $\mu_{\max}=12.0$. To quantify the trade-off, we train two additional GCAs under the same training setup (hyperparameters and computing power-hardware) as GC-LSN: \textbf{GC-LSN-Small}, with $p_{\max}=50.0$, and \textbf{GC-LSN-Large}, with $p_{\max}=200.0$ and $\mu_{\max}=25.0$. These agents are then tested on the same benchmark testbed (Cases 1, 2, 3). Training required \(\sim 12\) hours for GC-LSN-Small and \(\sim 24\) hours for GC-LSN-Large, with the differences arising from the larger action space, which increases both the simulation budget and the size of the network’s output layer.

We report the detailed results in Appendix~\ref{sec:ablation}. GC-LSN-Large performs slightly worse than GC-LSN but still outperforms C-BSP, except in stochastic lead times with order crossover. This performance could likely be improved with a larger compute budget, and since training is a one-time cost, adopting larger parameterizations at the expense of compute power can still be favorable in practice. In contrast, GC-LSN-Small achieves better performance in cases ($p=39.0$) fully contained in its smaller $\hat{\parSpace}$, since the sampled training instances are closer together, making generalization easier. However, when tested on instances outside this restricted $\hat{\parSpace}$, performance deteriorates. This illustrates the importance of choosing sufficiently large bounds, particularly when little is known about the true parameter ranges.

To measure the extent to which generalization occurs, we next evaluate the robustness of GC-LSN when the true parameters fall outside $\hat{\parSpace}$. Figure~\ref{fig:deviation} reports the relative cost gap between GC-LSN and C-BSP as we move the mean demand $\mu$ and penalty cost $p$ away from the training bounds $p_{\max}=100.0$ and $\mu_{\max}=12.0$. Within the training range, GC-LSN consistently outperforms C-BSP. Moreover, GC-LSN retains a superior performance for moderate deviations, e.g., up to $\mu=20.0$ and $p=250.0$. As $\mu$ or $p$ move further away, the relative cost gap increases, and GC-LSN’s performance degrades, eventually no longer improving over C-BSP. The deterioration in performance as $p$ increases is expected, since C-BSP is asymptotically optimal in $p$ \citep{Xin2021}.

\begin{figure}[ht]
\centering
\includegraphics[width=1.0\textwidth]{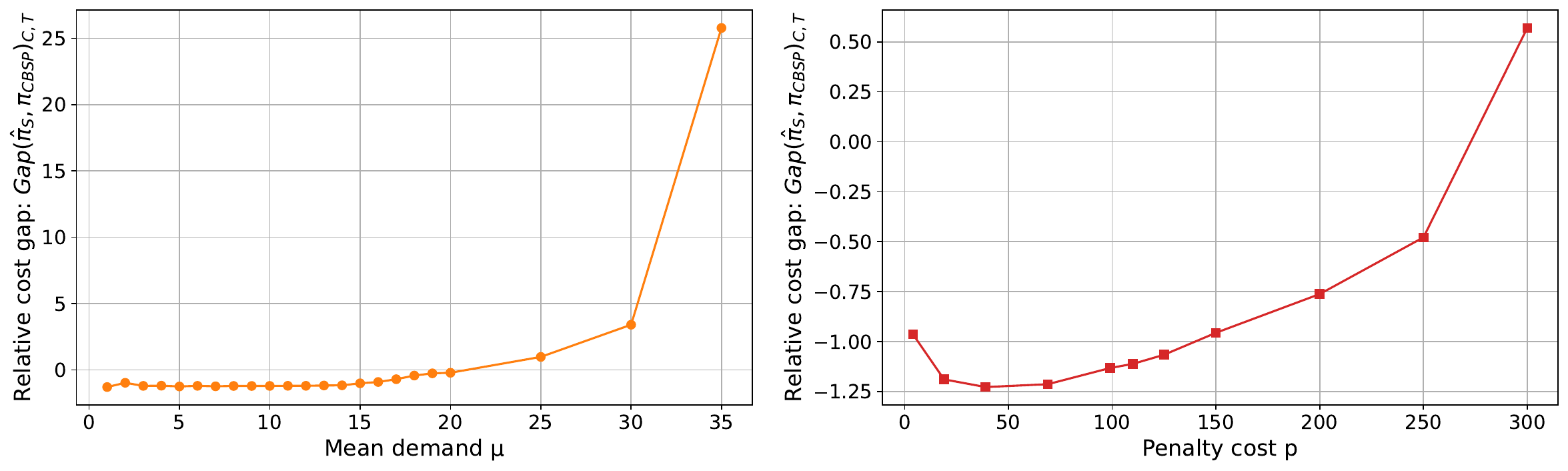}
\caption{Relative cost gap—lower is better; \% omitted—between GC-LSN and C-BSP for a Case~1 parametrization ($K=1$, geometric distributed demand, lead time is 8)  as the mean demand $\mu$ and penalty cost $p$ move beyond the training bounds $\mu_{\max}=12.0$ and $p_{\max}=100.0$.}
\label{fig:deviation}
\end{figure}

To gain more insights, we systematically vary the penalty cost $p \in \{1.0, 110.0, 125.0, 150.0\}$ and mean demand $\mu \in \{1.0, 13.0, 15.0, 20.0\}$, and average performance over the remaining Case~1 parameters to isolate the effects of $p$ and $\mu$ and reduce variability due to other parameter choices. Detailed numerical results for GC-LSN, GC-LSN-Large, and GC-LSN-Small are provided in Appendix~\ref{sec:ablation}. We confirm that GC-LSN retains good performance for mild deviations, but its effectiveness diminishes as parameters move further from the training range, and it underperforms C-BSP on average when $p=150.0$ or $\mu=15.0$. As expected, GC-LSN-Large performs better than GS-LSN on these instances since they fall within the expanded training range of GC-LSN-Large. For discrete parameters such as lead time and cycle length, robustness is more restrictive because they directly affect the featurization of the policy; we discuss these limitations in Appendix~\ref{sec:features}. We therefore conclude that GC-LSN exhibits reasonable robustness to mild deviations in the continuous parameters $p$ and $\mu$, although performance deteriorates as the true parameters move further away from the training range.

\vspace{-0.35cm}
\section{Conclusion}\label{sec:conclusion}
\vspace{-0.35cm}

In this study, we develop a solution framework designed to train GCAs that can make effective real-time decisions in inventory management problems. These agents are trained to handle ZSG settings, where estimates about the environmental factors might change unpredictably; yet, the agents do not require retraining. We introduce the concept of \emph{Super-MDPs} to address the formulation challenges encountered in these settings. Our \emph{TED} framework trains such GCAs and allows them to adapt in real-time for effective decision-making under ZSG. In numerical experiments focusing on the periodic review inventory control problem, our \emph{GC-LSN} policy outperforms the benchmarks.

From an industrial perspective, TED offers a viable pathway toward embedding DRL-based control into real-world inventory systems without incurring the retraining costs that currently limit deployment. By open-sourcing both the framework and trained weights, we lower the barrier for practitioners to experiment with and adopt state-of-the-art DRL solutions in their operations.

From an academic perspective, several research directions emerge. GC-LSN can serve as a ready benchmark in diverse inventory settings, particularly for evaluating new online learning methods in lost sales systems with cyclic demand and positive or stochastic lead times. Further work can also explore experimental cases for GC-LSN where demand distributions evolve but remain unknown, requiring alternative estimators beyond Kaplan–Meier (e.g., ARIMA-type processes). In addition, TED can be extended to other problem classes, such as service level agreement settings (cf. Appendix \ref{sec:keeplip}), where ensuring the Lipschitz property becomes more nuanced.

\vspace{-0.35cm}
\section*{Acknowledgements}
\vspace{-0.35cm}
Tarkan Temizöz conducted his research in the project DynaPlex: Deep Reinforcement Learning for Data-Driven Logistics, made possible by TKI Dinalog and the Topsector Logistics. We acknowledge the support of the SURF Cooperative using grant no. EINF-11324. 

\vspace{-1cm}
\begin{singlespace}
\bibliographystyle{apalike}
\bibliography{ref}
\end{singlespace}
\newpage
\appendix

\vspace{-0.35cm}
\section{Proof of Theorem \ref{thm:estimate}} \label{sec:proof}
\vspace{-0.35cm}

This theorem is a result of the Simulation Lemma in \citet{Kearns2002NearOptimalRL} for the Super-MDPs. The Simulation Lemma is a foundational result that provides a bridge between the behavior of a reinforcement learning policy in a simulated environment and its performance in the true environment. In our context, it implies that if the estimated parameters closely approximate the true parameters, any policy that is optimal (or near-optimal) in the true environment will also be near-optimal in the simulated environment. Therefore, our GCA can achieve zero-shot generalization, performing effectively even without precise knowledge of the true parameters.

\begin{proof} \emph{Bound on Cost Difference Due to Estimation Error.}

Our proof follows from applying the Simulation Lemma in the undiscounted, finite-horizon setting of \citet{lobel2024optimaltightnessboundsimulation} to the family of induced MDPs within the Super–MDP $\mathcal{M}_S$. They define two MDPs, $\mathcal{M} = \langle \mathcal{S}, \mathcal{A}, f, C \rangle$ and $\hat{\mathcal{M}} = \langle \mathcal{S}, \mathcal{A}, \hat{f}, \hat{C} \rangle$, which share a state-action space but differ in their transition and reward functions. Without loss of generality, we extend this concept by defining a Super-MDP $\mathcal{M}_S$ that encompasses these MDPs under distinct parameterizations $\Pars$ and $\hat{\Pars}$, hence yielding different cost and transition dynamics. 

We provide our proof in 3 steps.

\noindent\textbf{Step 1.}  

Denote $V^\pi_{T}(\s,\Pars)$ as the expected cumulative cost obtained by following policy $\pi$ from state $\s$ with parameterization $\Pars$ over $T$ time steps. Define the hybrid sequences:
\[
\mathbf{q}^{(t)} := \bigl(\hat{\Pars}_1,\ldots,\hat{\Pars}_t,\underbrace{\Pars,\ldots,\Pars}_{t+1:\,T}\bigr),
\qquad
V_T^\pi(\s; t) := V_T^\pi\bigl(\s;\mathbf{q}^{(t)}\bigr),
\]
so that $V_T^\pi(\s;0)=V_T^\pi(\s,\Pars)$ and $V_T^\pi(\s;T)=V_T^\pi(\s,\hat{\Pars}_{1:T})$.
Introduce the one–step differences:
\[
\Delta_t \;:=\; V_T^\pi(\s; t)-V_T^\pi(\s; t-1), \qquad t=1,\ldots,T.
\]
By telescoping, we have
\[
V_T^\pi(\s,\hat{\Pars}_{1:T})-V_T^\pi(\s,\Pars)
= \sum_{t=1}^{T}\!\bigl[V_T^\pi(\s; t)-V_T^\pi(\s; t-1)\bigr]
= \sum_{t=1}^{T}\Delta_t.
\]
Taking absolute values and applying the triangle inequality yields
\begin{equation}\label{eq:tel}
\bigl|V_T^\pi(\s,\hat{\Pars}_{1:T})-V_T^\pi(\s,\Pars)\bigr|
= \Bigl|\sum_{t=1}^{T}\Delta_t\Bigr|
\;\le\; \sum_{t=1}^{T}\bigl|\Delta_t\bigr|
= \sum_{t=1}^{T}\bigl|V_T^\pi(\s; t)-V_T^\pi(\s; t-1)\bigr|.
\end{equation}

\noindent\textbf{Step 2.}  

Fix $t \in \{1, \ldots, T\}$ and set the effective horizon $H := T - t + 1$. The functions $V_T^\pi(\s; t)$ and $V_T^\pi(\s; t-1)$ are identical (in terms of the MDP they belong to) except at step $t$: the former uses $(C^{\hat{\Pars}_t}, f^{\hat{\Pars}_t})$ at time $t$, the latter uses $(C^{\Pars}, f^{\Pars})$. Define the $(H-1)$-step total expected future cost under the true parameterization after step $t$,  
\[
W_{H-1}^\pi(\s',\Pars) := \mathbb{E} \left[ \sum_{i=t+1}^T C^{\Pars}\big(\s_i,\pi(\s_i, \Pars)\big) \ \bigg| \ \s_{t+1} = \s' \right].
\]

Expanding one step, conditioning on $\s_t$ and $a_t \sim \pi(\cdot \mid \s_t, \Pars_t)$, and subtracting gives  
\[
\left| V_T^\pi(\s; t) - V_T^\pi(\s; t-1) \right| \le \underbrace{\mathbb{E} \Big[\left| C^{\hat{\Pars}_t}(\s_t, a_t) - C^{\Pars}(\s_t, a_t) \right|\Big]}_{\le \ \epsilon_r(\Pars, \hat{\Pars}_t)}  + \underbrace{\mathbb{E} \Big[\left| \mathbb{E}_{f^{\hat{\Pars}_t}} [W_{H-1}^\pi] - \mathbb{E}_{f^{\Pars}} [W_{H-1}^\pi] \right|\Big]}_{\text{transition term}},
\]
where $\epsilon_r(\Pars, \hat{\Pars}_t) := \max_{\s,a} \left| C^{\Pars}(\s,a) - C^{\hat{\Pars}_t}(\s,a) \right|.$ The term $\mathbb{E}  \Big[ \left| C^{\hat{\Pars}_t}(\s_t, a_t) - C^{\Pars}(\s_t, a_t) \right| \Big]$ denotes the expected absolute difference in the immediate cost at step $t$ when evaluated under the estimated cost function $C^{\hat{\Pars}_t}$ versus the true cost function $C^{\Pars}$. The term $\mathbb{E} \Big[\left| \mathbb{E}_{f^{\hat{\Pars}_t}} [W_{H-1}^\pi] - \mathbb{E}_{f^{\Pars}} [W_{H-1}^\pi] \right|\Big]$ denotes the expected absolute difference in the $(H-1)$ step total expected future cost when the next-state distribution at step $t$ is generated using the transition function under estimated parameterization $f^{\hat{\Pars}_t}$ instead of the true one $f^{\Pars}$.

For the transition term, apply the total variation inequality.  
\[
\left| \mathbb{E}_P [g] - \mathbb{E}_Q [g]\right| \le \frac{1}{2} \| P - Q \|_1 \cdot (\sup g - \inf g)
\]
with $g = W_{H-1}^\pi$. Since costs are non-negative and bounded (cf. Definition \ref{def:mdp}), $0 \le W_{H-1}^\pi \le (H-1) C_{\max}(\Pars, \hat{\Pars}_t),$
hence  
\[
\left| \mathbb{E}_{f^{\hat{\Pars}_t}} [W_{H-1}^\pi] - \mathbb{E}_{f^{\Pars}} [W_{H-1}^\pi] \right| \le \frac{H-1}{2} \, C_{\max}(\Pars, \hat{\Pars}_t) \, \left\| f^{\hat{\Pars}_t}(\cdot \mid \s_t, a_t) - f^{\Pars}(\cdot \mid \s_t, a_t) \right\|_1.
\]

Maximizing over $(\s_t, a_t)$ yields
\[
\left| V_T^\pi(\s; t) - V_T^\pi(\s; t-1) \right|\le \epsilon_r(\Pars, \hat{\Pars}_t) + \frac{H-1}{2} \, \epsilon_f(\Pars, \hat{\Pars}_t)= \epsilon_r(\Pars, \hat{\Pars}_t) + \frac{T-t}{2} \, \epsilon_f(\Pars, \hat{\Pars}_t),
\]
where
\[
\epsilon_f(\Pars, \hat{\Pars}_t) := \max_{\s,a} C_{\max}(\Pars,\hat{\Pars}_t) \left\| f^{\Pars}(\cdot \mid \s,a) - f^{\hat{\Pars}_t}(\cdot \mid \s,a) \right\|_1.
\]

\noindent This is exactly the single-step expression of the finite-horizon simulation bound (cf. Proposition~B.1 of \citet{lobel2024optimaltightnessboundsimulation}); compare their $(H-1)/2$ factor.

If, moreover, the chain induced by $(f^{\Pars},\pi)$ is uniformly ergodic uniformly over $\Pars\in\hat{\parSpace}$, then the range of the $(H-1)$-step value is uniformly bounded \citep{ortner2024note}: there exists $\tilde{\kappa}(\pi)<\infty$ such that
\[
\sup_{\s'} W_{H-1}^{\pi}(\s',\Pars)-\inf_{\s'} W_{H-1}^{\pi}(\s',\Pars)\ \le\ \tilde{\kappa}(\pi)\qquad\text{for all $H$ and all $\Pars\in\hat{\parSpace}$}.
\]
Applying the same total-variation inequality with this bound replaces $(H-1)$ by $\tilde{\kappa}(\pi)$, i.e.,
\[
\left| \mathbb{E}_{f^{\hat{\Pars}_t}} [W_{H-1}^\pi] - \mathbb{E}_{f^{\Pars}} [W_{H-1}^\pi]\right|\ \le\ \tfrac{\tilde{\kappa}(\pi)}{2} C_{\max}(\Pars, \hat{\Pars}_t)\,\left\| f^{\hat{\Pars}_t}(\cdot \mid \s_t, a_t) - f^{\Pars}(\cdot \mid \s_t, a_t) \right\|_1,
\]
so the bound in this step becomes $\left| V_T^{\pi}(\s; t) - V_T^{\pi}(\s; t-1) \right| \le \epsilon_r(\Pars, \hat{\Pars}_t) + \tfrac{\tilde{\kappa}(\pi)}{2}\, \epsilon_f(\Pars, \hat{\Pars}_t),$ uniformly in $t$ and $T$.

\noindent\textbf{Step 3.}  

Combining Step 1 and Step 2,
\begin{equation} \label{eq:core-sum}
\big| V_T^\pi(\s, \hat{\Pars}_{1:T}) - V_T^\pi(\s, \Pars) \big| \le \sum_{t=1}^T \left(\epsilon_r(\Pars, \hat{\Pars}_t) + \frac{T-t}{2} \, \epsilon_f(\Pars, \hat{\Pars}_t) \right).
\end{equation}

By the definition of the parameterization distance (Definition \ref{def:param}) $d_{P}$,
\[
\epsilon_r(\Pars, \hat{\Pars}_t) \le d_{P}(\Pars, \hat{\Pars}_t), \quad \epsilon_f(\Pars, \hat{\Pars}_t) \le d_{P}(\Pars, \hat{\Pars}_t).
\]
Substituting into \eqref{eq:core-sum} and dividing by $T$ to obtain per period expected cost difference gives
\[
\left| \bar{C}_T^\pi(\s, \hat{\Pars}_{1:T}) - \bar{C}_T^\pi(\s, \Pars) \right| \le \frac{1}{T} \sum_{t=1}^T \left( 1 + \frac{T-t}{2} \right) d_{P}(\Pars, \hat{\Pars}_t),
\]
as claimed.

Under uniform ergodicity,
\[
\big| V_T^{\pi}(\s, \hat{\Pars}_{1:T}) - V_T^{\pi}(\s, \Pars) \big|\le \sum_{t=1}^T \!\left(\epsilon_r(\Pars,\hat{\Pars}_t) + \tfrac{\tilde\kappa(\pi)}{2}\,\epsilon_f(\Pars,\hat{\Pars}_t)\right).
\]
Dividing by $T$ and using $\epsilon_r,\epsilon_f\le d_P$ gives
\[
\left| \bar{C}_T^{\pi}(\s, \hat{\Pars}_{1:T}) - \bar{C}_T^{\pi}(\s, \Pars) \right|\le \frac{1}{T} \sum_{t=1}^T \left( 1 + \frac{\tilde\kappa(\pi)}{2} \right) d_{P}(\Pars, \hat{\Pars}_t)= \frac{\kappa(\pi)}{T}\sum_{t=1}^T d_{P}(\Pars,\hat{\Pars}_t),
\]
with $\kappa(\pi):=1+\tilde\kappa(\pi)/2$ .
\qedhere
\end{proof}

\vspace{-0.35cm}
\section{Robust optimization for Decide phase}\label{sec:robust-decide}
\vspace{-0.35cm}

We define the robust optimization model $\mathcal{Q}$, where the objective is to select a parameterization $\hat{\Pars}$ that minimizes the expected cumulative cost over a horizon $T$ given the worst-possible parameter realization $\Pars$, assuming no initial observations are available and a given GCA $\hat{\pi}_S$ is used. Here, $T$ covers the periods where we do not have any observation to estimate the parameters. For instance, we can only observe the demands after lead time if there is no stock available initially. The detailed formulation of $\mathcal{Q}$ is as follows:

\begin{itemize}
    \item \textbf{Decision Variable}:
    \[
    \hat{\Pars} \in \hat{\parSpace}
    \]
    
    \item \textbf{Objective Function}:
    \[
    \min_{\hat{\Pars} \in \hat{\parSpace}} \max_{\Pars \in \hat{\parSpace}} \mathbb{E}\left[ \sum_{t=0}^{T-1} C(\s_t, \pi(\s_t, \hat{\Pars})) \mid \Pars,\hat{\pi}_S \right]
    \]
    where $C(\s_t, \hat{\pi}_S(\s_t, \hat{\Pars}))$ represents the cost at time $t$ with state $\s_t$, parameter estimate $\hat{\Pars}$ and action $\hat{\pi}_S(\s_t, \hat{\Pars})$, given the true parameterization $\Pars$ and a policy $\pi$. 
    \item \textbf{Constraints}:
    \[
    \s_{t+1} = f^{\Pars}(\s_t, \hat{\pi}_S(\s_t, \hat{\Pars})) 
    \]
    with $\s_t \in \mathcal{S}$ and $\hat{\pi}_S(\s_t, \hat{\Pars}) \in \mathcal{A}$ for all $t \in \{0, \dots, T-1\}$, ensuring that actions taken are feasible for the estimated parameters.
    
\end{itemize}

This robust optimization model $\mathcal{Q}$ is designed to mitigate risk in the face of uncertainty about the true parameterization $\Pars$ during initial decision periods without any observations. Therefore, solution to $\mathcal{Q}$ can guide us when we lack parameter estimates due to the absence of data.

\vspace{-0.35cm}
\section{Train phase for periodic review inventory control}\label{sec:smdplostsales}
\vspace{-0.35cm}

We first explain the details for constructing a Super-MDP $\hat{\mathcal{M}}_S = (\hat{\parSpace}, \mathcal{S}, \mathcal{A}, \hat{\mathcal{H}}, \mathcal{F})$ for the periodic review inventory control problem. We then define the initial policy used in the Super-DCL algorithm. 

\noindent \textbf{Probable parameter space} $\hat{\parSpace}$

Recall that the parameters $\Pars$ related to the inventory problem discussed in \S\ref{sec:inventorymodel} are: $\Pars = (h, p, \boldsymbol\zeta, \boldsymbol\tau)$, where $h$ denotes the holding costs, $p$ the penalty costs, $\boldsymbol\zeta$ the specifics of the demand distribution, and $\boldsymbol\tau$ the specifics of the lead time distribution. The probable parameter space, denoted as $\hat{\parSpace}$, encompasses all possible combinations of these elements.

For the cost-related parameters $h$ and $p$, we maintain $h$ as a constant, set to $h=1$, while $p$ is variable, ranging within $[p_{min}, p_{max}]$, where both $p_{min}$ and $p_{max}$ are predefined scalars. 

Concerning the demand distribution, we adopt a structure capable of accommodating the cyclic nature of demand. Define $j$ as the subset of demands generated from the same distribution in the cycles, where $j = t \mod K$, $K$ is the cycle length, $K < K_{max}$, and $t$ is the time period. The mean demand for each subset is denoted by $\mu_j$, ranging from $[\mu_{min}, \mu_{max}]$, and the standard deviation of the demand by $\sigma_j$, which varies between $[\sigma_{min}(\mu_j), \mu_j \cdot 2]$. The demand distributions are assumed to be fitted based on the first two moments, following the methodology specified by \citet{adan1995}. The function $\sigma_{min}(\mu_j)$ provides the minimum feasible standard deviation for fitting such a distribution depending on the mean demand \citep[see][]{vanhezewijk2024nonnegative} while $\mu_j \cdot 2$ represents the upper limit, indicative of a highly variable distribution. To keep the support bounded, we truncate each fitted demand distribution by discarding outcomes whose probability is below a small threshold $\varepsilon>0$ and renormalize. Let $\mathcal{D}_{\mu_j, \sigma_j}$ denote the fitted demand distribution. The realized demand at period $t$ is distributed according to $\mathcal{D}_{\mu_j, \sigma_j}$ with $j = t \mod K$. We construct $\boldsymbol\zeta$ to encapsulate the mean and standard deviation of the demand for each cycle period, defined as $\boldsymbol\zeta = (\mu_0, \sigma_0, \dots, \mu_{K-1}, \sigma_{K-1})$. This structure ensures that $\boldsymbol\zeta$ captures the sufficient statistics of the cyclic demand distribution. 

We assume the lead time for a placed order takes values between $0$ and $L_{max}$. Let $l$ denote whether the orders cross or not, such that $l =  \mathbf{1}_{\{ordercrossing\}}$. Let $p_j$ denote the probability of an order placed in period $t$ will be received in period $t+j$, and we define $\mathbf{L} = \{p_{L_{max}},p_{L_{max}-1} \dots, p_0\}$, such that $\sum_{j=0}^{j=L_{max}} p_j = 1.0$. We construct the lead time parameter as $\boldsymbol\tau = (l, \mathbf{L})$. 

\noindent \textbf{Parameter generation function} $\hat{\mathcal{H}}$

We now focus on the parameter generation function $\hat{\mathcal{H}}$, which is used each time a new parameterization is generated, see Line \ref{alg1:paramgeneration} of the Algorithm \ref{alg:DCL}. Our aim is to create a comprehensive function that can uniformly cover the probable parameter space $\hat{\parSpace}$, hence the true parameter space $\parSpace$. We assume a uniform distribution for penalty cost $p$, demand cycle length $K$, and the mean demand and standard deviation of the demand for each subset of a cycle, $\mu_j$ and $\sigma_j$, such that $p \sim U(p_{min}, p_{max})$, $K \sim U[1, K_{max}]$, and $\mu_j \sim U(\mu_{min}, \mu_{max})$ and $\sigma_j \sim U(\sigma_{min}(\mu_j), \mu_j \cdot 2)$. 

For the lead time distribution, we first determine minimum lead time $l_{min}$ and maximum lead time $l_{max}$, such that $l_{min} \sim U[0, L_{max}]$ and $l_{max} \sim U[l_{min}, L_{max}]$. In case, $l_{min} = l_{max}$, the lead time is deterministic. If the lead times are stochastic, we assume order crossing with probability $0.5$. We set $p_j=0.0$ for $j<l_{min}$ and $j>l_{max}$. We determine the remaining individual lead time probabilities in one of the following ways with equal probability: 1) we assign the probabilities uniformly to $p_{l_{min}},.., p_{l_{max}}$, 2) we fit a truncated discrete distribution between $p_{l_{min}}$ and $p_{l_{max}}$ using the same method we use for demand distribution \citep[see][]{adan1995}) with mean $(p_{l_{min}} +p_{l_{min}}) / 2$, 3) we assign random probabilities while ensuring that the total probabilities sum to $1.0$ and each of $p_j$, $l_{min} \leq j \leq l_{max}$, is assigned a value.

The configurations adapted for designing the Super-MDP for the periodic review inventory control problem align with existing literature \citep{zipkin2000foundations}. These configurations also facilitate the maintenance of the Lipschitz conditions (see Definition~\ref{def:lip}) by constructing a smooth and continuous probable parameter space $\hat{\parSpace}$ and through the use of a linear cost function. Moreover, employing a parameter generation function $\hat{\mathcal{H}}$ that uniformly covers $\hat{\parSpace}$ can further improve training performance and strengthen ZSG ability.

\noindent \textbf{State and action spaces}, $\mathcal{S}$ and $\A$

We define other Super-MDP elements for the periodic review inventory control model, and continue with state and action spaces, $\mathcal{S}$ and  $\A$. We first define what a state is composed of. Recall that $OH_t$ denotes the on-hand inventory at the beginning of period $t$. Let the vector $\mathbf{x_t}$ denote the pipeline inventory, the orders that have not been received yet, and $o_t$ denotes the number of to-be-received orders placed $t$ periods ago. So we can define $\mathbf{x_t}$ as $\mathbf{x_t} = \{o_{t-L_{max}}, o_{t-L_{max}-1}, \dots,  o_{t-1}\}$. Let $K_t=t \mod K$ denote the cycle of the demand distribution at period $t$. By augmenting the state space with a period index, we transform any induced MDP into a time-homogeneous (stationary) one \citep[see][Section 2.3]{kallenberg2011markov}. The tuple $(OH_t, \mathbf{x_t}, K_t)$ denotes the state of the system in the period $t$.

We construct a state space $\mathcal{S}$ covering all possible state spaces related to parameterizations $\mathcal{S}^\Pars$, such that $\mathcal{S}^\Pars \subseteq \mathcal{S}$. Specifically, we set $\mathcal{S} = \mathcal{S}^{\Pars_{max}}$, where $\Pars_{max}$ is the parameterization with the largest penalty cost, demand distribution with the largest mean and standard deviation, and longest lead time.  

Actions represent the amount ordered in the system. For the action space $\A = \{0, 1, \dots, m\}$, we consider the same parameterization $\Pars_{max}$, and the maximum order quantity $m$ corresponds to the single-period newsvendor fractile bound \citep[see][]{zipkin2008old}.

\noindent \textbf{State transitions, costs and initial state for a given parameterization}, $f^\Pars$, $C^\Pars$ and $s_0^\Pars$

We analyze how the state $\s_t$ transitions to a new state $\s_{t+1}$. We update the on-hand inventory as $OH_{t+1} = \max\{0, OH_t + q_t + q_t^0 - D_t\}$. Let $q_{t,j}$ denote the order received in the period $t$ that is placed $j$ periods earlier, $1\geq j \geq L_{max}$ and $q_t = \sum_{j=1}^{j=L_{max}} = q_{t,j}$. We update the pipeline vector $\mathbf{x}_{t+1} = \{o_{t-L{max}-1} - q_{t,L{max}-1}, o_{t-L{max}-2} - q_{t,L{max}-2}, \dots, o_{t-1} - q_{t,1}, a_t - q_t^0 \}$. We have $K_{t+1}=(t+1) \mod K$. Thus, we can write the state transition matrix as follows:
\begin{align*}
f^{\Pars}(\s_t, a_t) &= (OH_{t+1}, \mathbf{x}_{t+1}, K_{t+1}) \\
&= (\max\{0, OH_t + q_t + q_t^0 - D_t\}, \{o_{t-L{max}-1} - q_{t,L{max}-1}, \dots, o_{t-1} - q_{t,1}, a_t - q_t^0 \}, K_{t+1}).
\end{align*}
Determining the exact values of $q_t$ and $q_t^0$ requires some straightforward but cumbersome formulation; readers may refer to \citet{bai2023asymptotic} and \citet{ChenLyuYuanZhou2023}.

After receiving the orders, placing new orders, and observing the demand, the system incurs costs at the end of a period with the following formulation:
\begin{align}
C^{\Pars}(\s_t, a_t) = h \cdot OH_{t+1} + p \cdot (D_t-(OH_t+q_t+q_t^0))^+, \label{eq:cost}
\end{align}
where the values $OH_t$, $OH_{t+1}$, $q_t$, and $q_t^0$ inherently depend on $\s_t$ and $a_t$, and the demand $D_t$ is drawn from the distribution $\mathcal{D}_{\mu_j, \sigma_j}$, with $j = t \mod K$.

For the initial states $\s_0^\Pars$, we follow the online learning in inventory systems literature and consider deterministic initial states. Specifically, we assume the initial on-hand inventory $OH_0 = 0$ and that there is no order in the pipeline. The system starts with the first action taken after observing the empty inventory.

\noindent \textbf{Initial policy for the Super-DCL algorithm} $\pi_0$ 

The Super-DCL algorithm requires an initial policy as an input, which it will iteratively refine to approximate the optimal policy. Let $I_{\textrm{max}}^\Pars$ denote the newsvendor fractile for cumulative demand over lead time, which bounds the optimal inventory position \cite[see][]{zipkin2008old} under parameterization $\Pars$. Then $\pi_0$ is set to be a capped base-stock policy with a base-stock level $I_{\textrm{max}}^\Pars$ and a maximum order quantity $m^\Pars_\s$, which is the single-period newsvendor fractile bound \citep{zipkin2008old}.

Given the constructed Super-MDP formulation for the periodic review inventory control problem and the initial policy $\pi_0$, we can train a GCA using the Super-DCL algorithm. Readers may refer to Appendix \ref{sec:exsetup} for the set of hyperparameters and the neural network structure required by the algorithm.

\vspace{-0.35cm}
\section{Estimate - Decide phases for periodic review inventory control} \label{sec:estimatelostsales}
\vspace{-0.35cm}

We explain how the interaction with the environment takes place, how observations are collected and used for estimating the parameters, and how decisions are made with the trained GCA. To this end, we present a set of assumptions:

\begin{assumption}[Interaction with the Environment]\label{as:4} 
The following assumptions guide us through the Estimate and Decide phases for the numerical experiments on our inventory problem. They are also commonly made in theory and in practice; see the related online inventory literature in \S\ref{sec:literature}.
\begin{itemize}
\item We have direct information on holding cost $h$ and the penalty cost $p$, so there is no estimation process for them.
\item For parameters for which no direct information is available - the demand and lead time distributions, we initially have no observation.
\item We know the cycle length $K$ for the demands. 
\item We know whether the lead times are deterministic (if so it is assumed to be known and orders do not cross, $l = 0$) or stochastic. If stochastic, we know whether the orders cross or not. 
\item For each received order, we have full information on when that particular order is placed: we can observe the individual lead times. It means that the pipeline inventory vector $\mathbf{x}$ can be updated accordingly.  
\item At the start of a period, we place orders using only information observed up to that time; the decision does not depend on the yet-unrealized demand, and thus the available-stock (censoring) threshold is fixed before demand is realized \citep[cf.][]{huh2011adaptive}.
\end{itemize}
\end{assumption}

We start with the decision-making strategy when there are no observations, and hence, no parameter estimates are available. Following the guidelines in \S\ref{sec:decide}, we aim to have a robust approach and minimize the costs in the worst-case scenario. Without explicitly formulating and solving the robust optimization problem $\mathcal{Q}$, presented in Appendix \ref{sec:robust-decide}, we greedily adopt a parameterization with deterministic demands equal to $\mu_{max}$ until the first sale or positive on-hand inventory (in case the demand distribution is unknown) and the largest deterministic lead time $L_{max}$ until we receive the first order arrival (in case the lead time distribution is unknown). This parameterization worked well in numerical experiments.

When we start collecting observations, related unknown parameters are estimated. For the unknown demand distribution, we utilize the Kaplan-Meier estimator \citep{kaplan1958nonparametric}. While initially proposed for estimating the survival function from lifetime data, the Kaplan-Meier estimator is also adopted in inventory problems. Interested readers may refer to \citet{huh2011adaptive, lyu2024} for the usage of it to build an empirical cumulative distribution function (CDF) from the censored demand observations. While we depart from them in having a cyclic demand distribution, the procedure stays the same as we know the cycle length $K$, see Assumption \ref{as:4}. So we repeat the process for each subset of demands within a cycle. We then derive the estimated mean demand $\hat{\mu}$ and the standard deviation of the demand $\hat{\sigma}$ from the empirical CDF to create and update the demand distribution estimate $\hat{\boldsymbol\zeta} = (\hat{\mu_0}, \hat{\sigma_0}, \dots, \hat{\mu}_{K-1}, \hat{\sigma}_{K-1})$.

The estimation for the lead time probabilities is rather straightforward. In particular, we build an empirical relative frequency distribution regarding how many periods it takes for an order to be received, and calculates the corresponding estimates for lead time probabilities $\hat{\mathbf{L}} = \{\hat{p}_{L_{max}},\hat{p}_{L_{max}-1} \dots, \hat{p}_0\}$. This estimation process is the same regardless of whether the order crosses or not, as that information is contained in $l$ of the parameter $\hat{\boldsymbol\tau}=(l,\hat{\mathbf{L}})$.

\subsection*{Consistency of the estimators}

\begin{proposition}[Kaplan–Meier Convergence for Censored, $K$-cyclic Demand]\label{prop:km-consistency}
Assume $K$-cyclic i.i.d. demand within each cycle, where each period’s decision uses only past orders (last point of Assumption~\ref{as:4}). Let $\widehat F_{t,j}$ be the Kaplan–Meier estimator for cycle $j$, and $F_{j}$ be the true CDF. Then, $\sup_{x\in \{0,\dots,D_{\max}\}  }|\widehat F_{t,j}(x)-F_j(x)| \xrightarrow{\text{a.s.}} 0$.
\end{proposition}
\begin{proof}
Fix $j$. Within cycle $j$ the demands are i.i.d.\ (discrete) and, by Assumption~\ref{as:4}, the censoring level $OH_t+q_t+q_t^0$ is chosen from past information, so censoring is independent of the demand. Therefore, the conditions of Theorem~2 in \citet{huh2011adaptive} hold for cycle $j$, which yields $\sup_{x\in \{0,\dots,D_{\max}\} }|\widehat F_{t,j}(x)-F_j(x)| \to 0$ almost surely. 
\end{proof}

\begin{proposition}[Relative-Frequency Distribution for Lead Times]\label{prop:lt-consistency}
Let $\widehat{\mathbf L}_t$ be the empirical relative-frequency estimator of $\mathbf L$. Then $\|\widehat{\mathbf L}_t-\mathbf L\|_1\to 0$ almost surely as $t\to\infty$.
\end{proposition}
\begin{proof}
Since there are finitely many $\ell\in\{0,\ldots,L_{\max}\}$, the empirical frequencies converge almost surely to the true probabilities for each $\ell$: $\widehat p_{t,\ell} \xrightarrow{\text{a.s.}} p_\ell$, by the Strong Law of Large Numbers.
\end{proof}

\begin{corollary}[Consistency in $d_P$ from Propositions~\ref{prop:km-consistency}–\ref{prop:lt-consistency}]
\label{cor:dp-consistency}
Let $\Pars=(h,p,\boldsymbol\zeta,\boldsymbol\tau)$ and $\hat{\Pars}_t=(h,p,\hat{\boldsymbol\zeta}_t,\hat{\boldsymbol\tau}_t)$, with $h,p$ observed and $\hat{\boldsymbol\zeta}_t,\hat{\boldsymbol\tau}_t$ estimated as in Propositions~\ref{prop:km-consistency}–\ref{prop:lt-consistency}. The corresponding Super-MDP is Lipschitz in the sense of Definition~\ref{def:lip}. Then
\[
d_P(\Pars,\hat{\Pars}_t)\ \xrightarrow[t\to\infty]{\mathbb{P}}\ 0 .
\]
\end{corollary}
\begin{proof}
Because $h$ and $p$ are observed, the distance $d_P(\Pars,\hat{\Pars}_t)$ depends only on the demand and lead-time distributions $(\boldsymbol\zeta, \boldsymbol\tau)$ versus $(\hat{\boldsymbol\zeta}_t, \hat{\boldsymbol\tau}_t)$.

Proposition~\ref{prop:km-consistency} yields, for each cycle $j$, $\sup_{x\in\{0,\dots,D_{\max}\}}|\widehat F_{t,j}(x)-F_j(x)|\to0$ a.s.; on a finite grid this implies the pmf converges coordinatewise, hence by the continuous mapping theorem $(\widehat\mu_{t,j},\widehat\sigma_{t,j})\to (\mu_j,\sigma_j)$ a.s., and therefore $\|\hat{\boldsymbol\zeta}_t-\boldsymbol\zeta\|_1\to0$ a.s. By Proposition~\ref{prop:lt-consistency}, $\|\hat{\boldsymbol\tau}_t-\boldsymbol\tau\|_1\to0$ a.s. By the Lipschitz Super-MDP property (Definition~\ref{def:lip}) and the definition of $d_P$ (Definition~\ref{def:param}), for all $(\s,a)$,
\[
|C^{\Pars}(\s,a)-C^{\hat{\Pars}_t}(\s,a)|\le L_r\Delta_t,\quad\|f^{\Pars}(\cdot|\s,a)-f^{\hat{\Pars}_t}(\cdot|\s,a)\|_1\le L_f\Delta_t,
\]
where $\Delta_t:=\|\hat{\boldsymbol\zeta}_t-\boldsymbol\zeta\|_1+\|\hat{\boldsymbol\tau}_t-\boldsymbol\tau\|_1$. Hence, by Definition~\ref{def:param},
\[
d_P(\Pars,\hat{\Pars}_t)\ \le\ \big(L_r+C_{\max}(\Pars,\hat{\Pars}_t)L_f\big)\,\Delta_t.
\]
Since $\Delta_t\to0$ a.s., the right-hand side $\to0$ a.s., and thus $d_P(\Pars,\hat{\Pars}_t)\to0$ a.s. (hence in probability).
\end{proof}

\vspace{-0.35cm}
\section{Experimental Setup}\label{sec:exsetup}
\vspace{-0.35cm}

Table~\ref{tab:lostsalessetup} reports the bounds we adopted for the probable parameter space of the periodic-review lost sales inventory control problem. These bounds are not arbitrarily chosen, and the success of the TED framework does not depend on their exact values provided sufficient compute is available for \emph{extensive sampling} and \emph{domain randomization}. The selected ranges are intended to reflect practically relevant environments while keeping the training of GCAs computationally feasible. At the same time, they should be viewed as representational choices for testing TED in numerical settings, rather than as strict prescriptions for real-world applications.

\begin{table}[ht]
  \centering
  \footnotesize
  \begin{tabular}{|c|c|c|c|c|c|c|}
    \bottomrule
$h$ & $p_{min}$ & $p_{max}$ & $\mu_{min}$ & $\mu_{max}$ & $K_{max}$ & $L_{max}$ \\
    \bottomrule
$1.0$ & $2.0$ & $100.0$ & $2.0$ & $12.0$ & $7$ & $10$\\
    \bottomrule
    \end{tabular}
\caption{Bounds for the probable parameter space for periodic review inventory control problem.}
\label{tab:lostsalessetup}
\end{table}

We normalize the holding cost to $h=1.0$, and set $p_{\min}=2.0$ since in lost sales systems, penalty costs are typically higher than holding costs; otherwise, holding no inventory would be favorable \citep[cf.][]{zipkin2008old,Xin2021}. We cap the penalty cost at $p_{\max}=100.0$ and lead time at $L_{max}=10$, as in regimes with very high penalty costs and long lead times, the capped base-stock policy is known to be asymptotically optimal \citep{Xin2021}, reducing the benefit of a sophisticated dynamic policy. For the demand distribution, we use $\mu_{\min}=2.0$ to avoid sparse demand levels, and set $\mu_{\max}=12.0$ for computational tractability, since the size of the action space grows with demand, and larger values would require extensive sampling (\S\ref{sec:train}) and significantly higher compute. A possible extension would be to introduce a low/high demand parameter (such as $l$ which shows whether orders cross or not), where in high-demand settings, a single action corresponds to a batch order with predefined quantities. We set $K_{\max}=7$ to reflect weekly demand cycles. 

Table \ref{tab:hyperparameters} presents the set of hyperparameters used to train the Super-DCL algorithm in our numerical experiments. For a comprehensive overview of the hyperparameters utilized, we direct readers to \citet{temizoz2025deep}. We employed the identical neural network architecture described in their study and adhered to the selection criteria for sampling and simulation-related hyperparameters outlined therein. To this end, we increased the number of samples $N$ and the number of approximate policy iterations $n$ for performance gains, and decreased the depth of the simulations $H$ and simulation budget per state-action pairs $M$ for computational efficiency. We refrained from further hyperparameter tuning. Departing from that study, we introduce two new hyperparameters: $R$, the number of samples collected for a parameterization; $P$, the maximum number of promising actions for simulations \citep[see][for an example usage]{danihelka2022policy}. 

\begin{table}[ht]
  \centering
  \footnotesize
    \begingroup
  \footnotesize
  \setlength{\tabcolsep}{3.5pt}
  \renewcommand{\arraystretch}{0.95}

  \begin{adjustbox}{max width=\textwidth}
  \begin{tabular}{|ll|ll|}
\midrule
\multicolumn{2}{|c|}{Sampling and Simulation} & \multicolumn{2}{c|}{Neural Network Structure $\NN_\theta$} \\
\midrule
Depth of the simulations:       &$H = 21$   & Number of layers:  & 4\\ 
Simulation budget per state-action pairs:    &$M = 500$ & Number of neurons: & \{ 256, 128, 128, 128 \}  \\
Number of samples:   &$N = 5000000$ & Optimizer:       & Adam \\
Length of the warm-up period:       &$L = 100$   &Mini-batch size:  & $MiniBatchSize = 1024$\\
Number of samples collected for a parameterization:     &  $R=100$ & Maximum epoch: & $100$\\
Maximum number of promising actions: & $P=16$ & Early stopping patience & $15$ \\
\midrule
\multicolumn{4}{|c|}{Number of approximate policy iterations: $n = 5$} \\
    \bottomrule
    \end{tabular}
\end{adjustbox}
\endgroup
\caption{The hyperparameters of the Super-DCL algorithm used in the experiments.}
\label{tab:hyperparameters}
\end{table}

\vspace{-0.35cm}
\section{Featurization of GC-LSN}\label{sec:features}
\vspace{-0.35cm}

Below we present the features for training GC-LSN. Using GC-LSN, we can take actions for any product at time $t$ with the estimated parameters. The features contain the following elements:

\begin{enumerate}
    \item $l$, whether the orders can cross. 
    \item $p$, penalty cost for each unmet demand.
    \item $OH_t+q_t$, on-hand inventory at the beginning of period $t$ + received orders in period $t$ before taking an action.
    \item Pipeline inventory vector, $\{o_{t-L{max}-1} - q_{t,L{max}-1}, o_{t-L{max}-2} - q_{t,L{max}-2}, \dots, o_{t-1} - q_{t,1} \}$.
    \item $\hat{\mathbf{L}}$, estimated lead time probabilities.
    \item $K$, demand cycle length.
    \item $\hat{\boldsymbol\zeta}_t^f$, features for estimated demand distribution statistics.    
\end{enumerate}

Note that features 5 and 7 are based on estimations. If the corresponding true parameters are known, we use them directly. Since the input size of the neural networks stays the same during training and demand cycle lengths $K$ vary across different parameterizations, we adapt $\hat{\boldsymbol\zeta}_t^f$, estimated demand distribution statistics (mean and standard deviation) for the next $K_{max}$ periods. Hence, regardless of the cycle length, $|\hat{\boldsymbol\zeta}_t^f|$ stays the same. This representation is analogous to augmenting the state with $K_t$ and directly including $\boldsymbol\zeta$ as part of the feature vector.

The bounds reported in Appendix \ref{sec:exsetup} directly affect this featurization, particularly through the discrete parameters. For instance, the maximum lead time $L_{\max}$ determines the length of the pipeline inventory vector and the estimated lead time probabilities $\hat{\mathbf{L}}$, while the maximum cycle length $K_{\max}$ supports the dimension of the estimated demand distribution features $\hat{\boldsymbol\zeta}_t^f$. Consequently, if a real-world instance exhibits longer lead times than $L_{\max}$ or cycle lengths beyond $K_{\max}$, GC-LSN cannot prescribe valid actions because the corresponding features are not represented in its input space. In contrast, for continuous parameters such as the penalty cost $p$ or mean demand $\mu$, GC-LSN can still map state–parameterization pairs that fall outside of $\hat{\parSpace}$, since these parameters only scale numerical inputs rather than change the structure of the feature vector. However, in such extrapolations, the performance is likely to deteriorate, as the policy has not been trained to generalize to those regions of the probable parameter space.

\vspace{-0.35cm}
\section{Instances of the Numerical Experiments}\label{sec:instances}
\vspace{-0.35cm}

Figure~\ref{fig:horizonlength} provides a representation of the probable parameter space and the instances used in Case 1 of our experiments. The horizontal axis shows the penalty cost $p$, while the vertical axis shows the mean demand $\mu$ of the geometric demand distribution, under the case of a deterministic lead time of 6. Blue dots indicate representative training samples, while orange stars mark the parameter settings used in the test set. This visualization is intended to illustrate how the sampled training instances cover the probable parameter space $\hat{\parSpace}$ and how the test instances are distributed within the same space. While the training samples are not dense, they are spread in a way that provides broad uniform coverage of $\hat{\parSpace}$.

\begin{figure}[ht]
\centering
\includegraphics[width=1.0\textwidth]{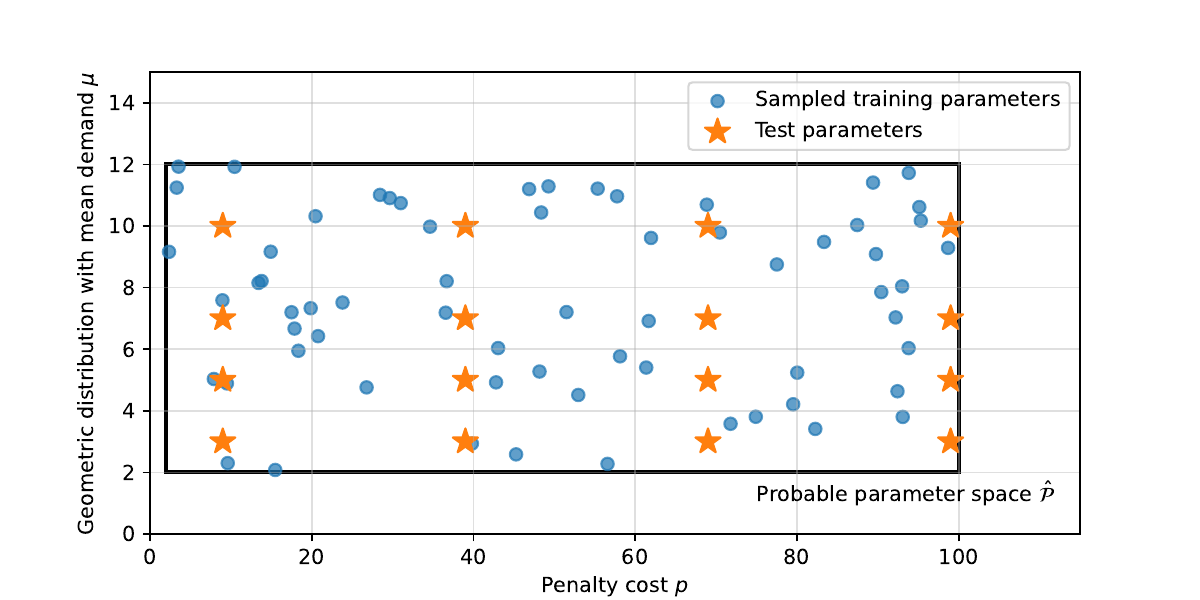}
\caption{Two-dimensional representation of the probable parameter space $\hat{\parSpace}$ for Case~1 instances with deterministic lead time of $6$ and geometrically distributed demand.}
\label{fig:horizonlength}
\end{figure}

\noindent \textbf{Instances of Case 2}

For Case 2, we have the following mean demands according to varying cycle length: 
\begin{itemize}
    \item For $K=3$, we have: $\{ \mu_1=2.5, \mu_2=4.5, \mu_3=3.0 \}, \{ \mu_1=9.0, \mu_2=11.0, \mu_3=9.5 \}, \{ \mu_1=3.0,\mu_2=6.0, \mu_3=10.0 \}$
    \item For $K=5$, we have: $\{ \mu_1=4.0, \mu_2=2.5, \mu_3=3.0, \mu_4=4.5, \mu_5=3.0 \}, \{ \mu_1=7.5, \mu_2=11.5, \mu_3=10.0, \mu_4=9.5, \mu_5=8.5 \}, \{ \mu_1=11.0, \mu_2=3.0, \mu_3=5.0, \mu_4=8.0, \mu_5=6.0 \}$
    \item For $K=7$, we have: $\{ \mu_1=3.0, \mu_2=4.0, \mu_3=2.0, \mu_4=3.5,\mu_5= 4.5, \mu_6=2.5, \mu_7=4.0 \}, \{ \mu_1=11.0,\mu_2= 10.0, \mu_3=9.0, \mu_4=10.0, \mu_5=11.0, \mu_6=8.5,\mu_7= 9.5 \}, \{ \mu_1=3.0, \mu_2=5.0, \mu_3=5.0, \mu_4=7.0, \mu_5=7.0, \mu_6=10.0, \mu_7=10.0 \}$
\end{itemize}

For each mean demand unit, we adapt three different standard deviation combinations: low, high, and mix. For each combination, the corresponding standard deviations are chosen according to the following distributions:  
\begin{itemize}
    \item For low standard deviation combination, we have: $\sigma_1-Poisson, \sigma_2-Binomial, \sigma_3-Poisson, \sigma_4-Binomial, \sigma_5-Poisson, \sigma_6-Binomial, \sigma_7-Poisson$.
    \item For high standard deviation combination, we have: $\sigma_1-Geometric, \sigma_2-NegativeBinomial, \sigma_3-Geometric, \sigma_4-NegativeBinomial, \sigma_5-Geometric, \sigma_6-NegativeBinomial, \sigma_7-Geometric$.
    \item For mixed standard deviation combination, we have: $\sigma_1-Poisson, \sigma_2-Geometric, \sigma_3-NegativeBinomial, \sigma_4-Binomial, \sigma_5-Poisson, \sigma_6-Geometric, \sigma_7-NegativeBinomial$.
\end{itemize}

We then fit a demand distribution by using the two moments as described by \citet{adan1995}.

\noindent \textbf{Instances of Case 3}

For Case 3, we consider the following demand distributions: for $K=1$, $\mu_1=5.0$; for $K=3$, $\mu_1=8.0, \mu_2=10.0, \mu_3=6.0 $; for $K=5$,  $\mu_1=11.0, \mu_2=3.0, \mu_3=5.0, \mu_4=8.0,\mu_5=6.0 $; for $K=7$, $ \mu_1=3.0, \mu_2=5.0, \mu_3=5.0, \mu_4=7.0, \mu_5=7.0, \mu_6=10.0, \mu_7=10.0 $. The corresponding standard deviations are chosen according to the following distributions: $\sigma_1-Poisson, \sigma_2-Geometric, \sigma_3-NegativeBinomial, \sigma_4-Binomial, \sigma_5-Poisson, \sigma_6-Geometric, \sigma_7-NegativeBinomial$. We then fit a demand distribution by using two moments as described by \citet{adan1995}.

When $l=0$, we have the following lead time probabilities $\mathbf{L}$:
\begin{itemize}
    \item $\mathbf{L} = \{0.0, 0.1, 0.18, 0.216, 0.2016, 0.1512, 0.09072, 0.042336, 0.0145152, 0.00326592, 0.00036288\}$
    \item $\mathbf{L} = \{0.0, 0.0, 0.2, 0.32, 0.288, 0.1536, 0.0384, 0.0, 0.0, 0.0, 0.0\}$
    \item $\mathbf{L} = \{0.0, 0.0, 0.0, 0.0, 0.1, 0.27, 0.378, 0.2268, 0.0252, 0.0, 0.0\}$
    \item $\mathbf{L} = \{0.0, 0.0, 0.0, 0.0, 0.0, 0.0, 0.0, 0.1, 0.27, 0.378, 0.252\}$
    \item $\mathbf{L} = \{0.0, 0.0, 0.0, 0.0, 0.0, 0.0, 0.2, 0.32, 0.288, 0.1536, 0.0384\}$
    \item $\mathbf{L} = \{0.0, 0.0, 0.0, 0.0, 0.05, 0.095, 0.171, 0.2052, 0.21546, 0.184338, 0.079002\}$
    \item $\mathbf{L} = \{0.0, 0.0, 0.0, 0.0, 0.1, 0.225, 0.3375, 0.253125, 0.0759375, 0.0084375, 0.0\}$
    \item $\mathbf{L} = \{0.05, 0.095, 0.171, 0.2052, 0.21546, 0.184338, 0.079002, 0.0, 0.0, 0.0, 0.0\}$
    \item $\mathbf{L} = \{0.0, 0.0, 0.0, 0.0, 0.0, 0.5, 0.25, 0.125, 0.0625, 0.03125, 0.03125\}$
    \item $\mathbf{L} = \{0.0, 0.0, 0.0, 0.3, 0.49, 0.189, 0.021, 0.0, 0.0, 0.0, 0.0\}$
\end{itemize}

When $l=1$, we have the following lead time probabilities $\mathbf{L}$:
\begin{itemize}
    \item $\mathbf{L} = \{ 0.0, 0.1, 0.1, 0.1, 0.1, 0.1, 0.1, 0.1, 0.1, 0.1, 0.1 \}$
    \item $\mathbf{L} = \{ 0.0, 0.0, 0.2, 0.2, 0.2, 0.2, 0.2, 0.0, 0.0, 0.0, 0.0 \}$
    \item $\mathbf{L} = \{ 0.0, 0.0, 0.0, 0.0, 0.1, 0.2, 0.3, 0.3, 0.1, 0.0, 0.0 \}$
    \item $\mathbf{L} = \{ 0.0, 0.0, 0.0, 0.0, 0.0, 0.0, 0.0, 0.1, 0.2, 0.3, 0.4 \}$
    \item $\mathbf{L} = \{ 0.0, 0.0, 0.0, 0.0, 0.0, 0.0, 0.2, 0.2, 0.2, 0.2, 0.2 \}$
    \item $\mathbf{L} = \{ 0.0, 0.0, 0.0, 0.0, 0.05, 0.05, 0.1, 0.1, 0.15, 0.25, 0.3 \}$
    \item $\mathbf{L} = \{ 0.05, 0.05, 0.1, 0.1, 0.15, 0.25, 0.3, 0.0, 0.0, 0.0, 0.0 \}$
    \item $\mathbf{L} = \{ 0.0, 0.0, 0.0, 0.0, 0.1, 0.15, 0.25, 0.25, 0.15, 0.1, 0.0 \}$
    \item $\mathbf{L} = \{ 0.0, 0.0, 0.0, 0.0, 0.0, 0.5, 0.0, 0.0, 0.0, 0.0, 0.5 \}$
    \item $\mathbf{L} = \{ 0.0, 0.0, 0.0, 0.3, 0.4, 0.2, 0.1, 0.0, 0.0, 0.0, 0.0 \}$
\end{itemize}

\vspace{-0.35cm}
\section{Per-Instance Training and Instance Selection}\label{sec:per-instance}
\vspace{-0.35cm}

We quantify the performance and operational trade-off when facing a choice between training and deploying a GCA with ZSG ability and training \emph{individual agents} tailored to fixed, known parameters. We therefore collect 35 instances spanning all three experimental cases and train individual DRL agents for each instance using the same neural network architecture and training pipeline as GC-LSN (with identical hyperparameters; only the number of training samples is adjusted to ensure per-instance convergence; the same hardware). 

For training individual DRL agents tailored to specific parameterizations, we employ the DCL algorithm \citep{temizoz2025deep}. We use the same hyperparameters reported in Table~\ref{tab:hyperparameters}, except for the number of samples, which is set to $N = 100000$. Note that the hyperparameters $R$ and $P$ are specific to the Super-DCL algorithm, which applies to training Super-MDPs and are therefore not considered here. Total training time for all agents was $\sim$ 6 hours, with each individual agent requiring less than 15 minutes to train.

The selected instances from the three experimental cases are as follows:

\textbf{Case 1:} We consider 16 problem instances using a full factorial design with the following parameter ranges: mean demand $\mu \in \{5.0, 10.0\}$, penalty cost $p \in \{9.0, 69.0\}$, and deterministic lead time $\in \{4, 8\}$. For each mean demand value, the standard deviation is selected so that the demand distribution follows either a Poisson or geometric distribution.

\textbf{Case 2:} We set $p=39.0$ and lead time to $6$. We consider 9 different instances with the same mean demands and varying cycle lengths reported in Appendix \ref{sec:instances}, where we employ \emph{mixed standard deviation combination}.  

\textbf{Case 3:} We set $p=69.0$, $K=3$ with mean demands $\mu_1=8.0, \mu_2=10.0, \mu_3=6.0 $ with Poisson, geometric, and negative binomial distributions, respectively. We consider 10 instances with varying $l$: $l=0$ and $l=1$. For both, we adopt the first five lead time distributions reported in Appendix \ref{sec:instances}.

We benchmark GC-LSN against these individual agents under known parameterizations. In summary, GC-LSN incurs a slightly higher cost on average, with a mean \emph{relative cost gap} of only $0.2\%$ versus per-instance agents under known parameterizations. Given this small gap, the retraining burden of per-instance models as tasks scale or parameters change, and the fact that per-instance training presumes known parameters (hence is infeasible when they are unknown), GC-LSN’s ZSG capability offers a practical advantage in settings with many tasks or frequent updates.

\vspace{-0.35cm}
\section{Restoring Lipschitz Property in Super-MDPs}\label{sec:keeplip}
\vspace{-0.35cm}

Consider a service level agreement that incurs a penalty $M>0$ whenever the (windowed) fill rate $\mathrm{FR}(\s,a,(\theta,\beta))$  falls below a threshold $\beta\in(0,1)$, where $\beta$ varies across items and is included in the parameterization vector $\Pars=(\theta,\beta)$, with $\theta$ denoting other parameters. The one-step cost is
\[
C^{(\theta,\beta)}(\s,a)\;=\;C^{\theta}(\s,a)\;+\;M\,\mathbf{1}\{\mathrm{FR}(\s,a,(\theta,\beta))<\beta\}.
\]
Fix $(\s,a,\theta)$ and choose $\beta'=\mathrm{FR}(\s,a,(\theta,\beta'))$ (e.g., by taking a state/action where the fill rate equals the threshold). Then, for any $\varepsilon>0$,
\[
\big|C^{(\theta,\beta')}(\s,a)-C^{(\theta,\beta'+\varepsilon)}(\s,a)\big| \;=\; M,\qquad d_{P}\big((\theta,\beta'),(\theta,\beta'+\varepsilon)\big)=\varepsilon,
\]
so $\sup_{\varepsilon\rightarrow 0}\, M/\varepsilon=\infty$: the mapping $p\mapsto C^{\Pars}$ is \emph{not} Lipschitz in $\beta$ (hence the Super–MDP is not Lipschitz) under any metric that treats $\beta$ as a standard continuous coordinate. Moreover, as $\beta\rightarrow 1$, achieving $\mathrm{FR}\ge \beta$ with probability close to one under demand models with unbounded support (e.g., Poisson, geometric) requires unbounded inventory levels; thus either the feasible state space or the cost remains unbounded when $\beta$ is near 1, which further undermines uniform continuity.

We can restore the Lipschitz property on a bounded domain in the following way. We treat $\beta$ as categorical (e.g., a finite grid $\{\beta_1,\ldots,\beta_Z\}$) and use one-hot encoding in $\Pars$; then cross-threshold pairs are never arbitrarily close, and $\big|C^{\Pars}-C^{\Pars'}\big|\le M\,\mathbf{1}\{\beta\neq\beta'\}\le M\,d_{P}(\Pars,\Pars')$, giving a finite Lipschitz constant.

If $\beta$ is truly continuous, however, then the probable parameter space $\hat{\mathcal P}$ can no longer be a strict superset of the true space $\mathcal P$: exact coverage would require infinitely many grid points. In practice, achieving adequate coverage of a continuous $\beta$ requires choosing a sufficiently fine discretization level. Larger $Z$ improves the approximation and reduces the looseness of the Lipschitz constant, but it also increases the dimensionality of the training space and hence the sampling budget needed to cover it uniformly. Thus, while discretization restores the Lipschitz property, it introduces a trade-off between fidelity to the true continuous space and computational feasibility.

\section{Ablation Experiment Results}\label{sec:ablation}

\begin{table}[ht]
   \centering
   \footnotesize
\begin{tabular}{r|r|r|r|r|r|r|r|r|r|}
    \cmidrule{2-10}
    \multicolumn{1}{c}{\multirow{1}[1]{*}{}} &\multicolumn{4}{|c|}{Case 1} & \multicolumn{3}{c|}{Case 2} & \multicolumn{2}{c|}{Case 3} \\
    \cmidrule{2-10}
    \multicolumn{1}{c|}{\multirow{1}[1]{*}{}} & \multicolumn{1}{c}{$p=9$} & \multicolumn{1}{c}{$p=39$} & \multicolumn{1}{c}{$p=69$} & \multicolumn{1}{c|}{$p=99$} & \multicolumn{1}{c}{$K=3$} & \multicolumn{1}{c}{$K=5$} & \multicolumn{1}{c|}{$K=7$} & \multicolumn{1}{c}{$l=0$} & \multicolumn{1}{c|}{$l=1$} \\
    \midrule
    GC-LSN & \multicolumn{1}{r}{$89.5\%$} & \multicolumn{1}{r}{$97.4\%$} & 
    \multicolumn{1}{r}{$98.5\%$} & 
    \multicolumn{1}{r|}{$99.0\%$} &
    \multicolumn{1}{r}{$93.8\%$} & 
    \multicolumn{1}{r}{$94.6\%$} &
    \multicolumn{1}{r|}{$94.2\%$} &
    \multicolumn{1}{r}{$93.5\%$} & 
    \multicolumn{1}{r|}{$94.6\%$} \\
   BSP& \multicolumn{1}{r}{$89.4\%$} & \multicolumn{1}{r}{$97.4\%$} & 
    \multicolumn{1}{r}{$98.6\%$} & 
    \multicolumn{1}{r|}{$99.0\%$} &
    \multicolumn{1}{r}{$93.8\%$} & 
    \multicolumn{1}{r}{$94.8\%$} &
    \multicolumn{1}{r|}{$94.3\%$} &
    \multicolumn{1}{r}{$93.7\%$} & 
    \multicolumn{1}{r|}{$94.9\%$} \\
  C-BSP & \multicolumn{1}{r}{$89.4\%$} & \multicolumn{1}{r}{$97.3\%$} & 
    \multicolumn{1}{r}{$98.5\%$} & 
    \multicolumn{1}{r|}{$99.0\%$} &
    \multicolumn{1}{r}{$93.8\%$} & 
    \multicolumn{1}{r}{$94.6\%$} &
    \multicolumn{1}{r|}{$94.3\%$} &
    \multicolumn{1}{r}{$93.5\%$} & 
    \multicolumn{1}{r|}{$95.0\%$} \\  
    \bottomrule    
\end{tabular}
    \caption{Service levels (fraction of demand met from stock) of policies.}
    \label{tab:servicelevels}
\end{table}

\begin{table}[ht]
   \centering
   \footnotesize
\begin{tabular}{r|r|r|r|r|r|r|r|r|r|}
    \cmidrule{2-10}
    \multicolumn{1}{c}{\multirow{1}[1]{*}{}} &\multicolumn{4}{|c|}{Case 1} & \multicolumn{3}{c|}{Case 2} & \multicolumn{2}{c|}{Case 3} \\
    \cmidrule{2-10}
    \multicolumn{1}{c|}{\multirow{1}[1]{*}{}} & \multicolumn{1}{c}{$p=9$} & \multicolumn{1}{c}{$p=39$} & \multicolumn{1}{c}{$p=69$} & \multicolumn{1}{c|}{$p=99$} & \multicolumn{1}{c}{$K=3$} & \multicolumn{1}{c}{$K=5$} & \multicolumn{1}{c|}{$K=7$} & \multicolumn{1}{c}{$l=0$} & \multicolumn{1}{c|}{$l=1$} \\
    \midrule
    GC-LSN & \multicolumn{1}{r}{$-0.9\%$} & \multicolumn{1}{r}{$-0.9\%$} & \multicolumn{1}{r}{$-0.7\%$} & \multicolumn{1}{r|}{$-0.5\%$} & \multicolumn{1}{r}{$-4.1\%$} & \multicolumn{1}{r}{$-4.1\%$} & \multicolumn{1}{r|}{$-4.4\%$} & \multicolumn{1}{r}{$-1.3\%$} & \multicolumn{1}{r|}{$-1.3\%$} \\
   GC-LSN-Small & \multicolumn{1}{r}{$-0.9\%$} & \multicolumn{1}{r}{$-1.0\%$} & \multicolumn{1}{r}{$-0.2\%$} & \multicolumn{1}{r|}{$2.9\%$} & \multicolumn{1}{r}{$-4.0\%$} & \multicolumn{1}{r}{$-4.0\%$} & \multicolumn{1}{r|}{$-4.5\%$} & \multicolumn{1}{r}{$-1.2\%$} & \multicolumn{1}{r|}{$-1.1\%$} \\
  GC-LSN-Large & \multicolumn{1}{r}{$-0.3\%$} & \multicolumn{1}{r}{$-0.7\%$} & \multicolumn{1}{c}{$-0.6\%$} & \multicolumn{1}{r|}{$-0.4\%$} & \multicolumn{1}{r}{$-3.5\%$} & \multicolumn{1}{r}{$-3.1\%$} & \multicolumn{1}{c|}{$-1.9\%$} & \multicolumn{1}{r}{$-1.0\%$} & \multicolumn{1}{r|}{$1.5\%$} \\   
    \bottomrule    
\end{tabular}
    \caption{Relative cost gap of GC-LSN policies compared to C-BSP - the lower the better}
    \label{tab:gc-lsn-pols}
\end{table}

\begin{table}[ht]
  \centering
  \begingroup
  \footnotesize
  \setlength{\tabcolsep}{3.5pt}
  \renewcommand{\arraystretch}{0.95}

  \begin{adjustbox}{max width=\textwidth}
  \begin{tabular}{|r|rrr|rrr|}
    \cmidrule{2-7}
    \multicolumn{1}{c|}{\multirow{1}[1]{*}{}}& $p=110$ & $p=125$ & $p=150$ & $\mu=13$ & $\mu=15$ & $\mu=20$ \\
    \midrule
    GC-LSN & $-0.4\%$ & $-0.2\%$ & $0.9\%$ & $-0.4\%$ & $0.1\%$ & $6.9\%$ \\
        GC-LSN-Large & $-0.4\%$ & $-0.3\%$ & $-0.2\%$ & $-0.5\%$ & $-0.4\%$ & $-0.2\%$ \\
    \bottomrule
  \end{tabular}
  \end{adjustbox}

  \caption{Relative cost gap of GC-LSN and GC-LSN-Large compared to C-BSP across selected parameters - lower is better.}
  \label{tab:sensitivity}
  \endgroup
\end{table}

\end{document}